\newif\ifarxiv
\arxivfalse

\documentclass{article}

\usepackage{fullpage}

\usepackage[utf8]{inputenc} 
\usepackage[T1]{fontenc}    
\usepackage{hyperref}       
\usepackage{url}            
\usepackage{booktabs}       
\usepackage{amsfonts}       
\usepackage{nicefrac}       
\usepackage{microtype}      
\usepackage{xcolor}         
\usepackage{amsthm,amsmath}
\usepackage[ruled]{algorithm2e}
\usepackage{bbm}
\usepackage{natbib}
\usepackage{color-edits}
\usepackage{algpseudocode}
\usepackage{comment}
\usepackage{graphicx}
\usepackage{thmtools,thm-restate}
\usepackage{authblk}
\title{Multicalibrated Regression for Downstream Fairness}

 \author[1]{Ira Globus-Harris}
 \author[1]{Varun Gupta}
 \author[2]{Christopher Jung}
 \author[1]{Michael Kearns}
 \author[3]{Jamie Morgenstern}
 \author[1]{Aaron Roth}
 \affil[1]{University of Pennsylvania}
 \affil[2]{Stanford University}
 \affil[3]{University of Washington}

\newcommand{\ar}[1]{\textcolor{blue}{[Aaron: #1]}}

\newcommand{\Ex}{\mathbb{E}}
\newcommand{\Cov}{\textrm{Cov}}
\newcommand{\cS}{\mathcal{S}}

\newcommand{\cB}{\mathcal{B}(C)}

\newcommand{\cD}{\mathcal{D}}
\newcommand{\cX}{\mathcal{X}}
\newcommand{\cC}{\mathcal{C}}
\newcommand{\cH}{\mathcal{H}}
\newcommand{\cG}{\mathcal{G}}

\newcommand{\cY}{\mathcal{Y}}

\newcommand{\err}{\text{err}}

\newcommand{\fn}{\text{FN}}
\newcommand{\fp}{\text{FP}}
\newcommand{\SP}{\text{SP}}

\newcommand{\E}{\mathop{\mathbb{E}}}
\newcommand{\opt}{\mathrm{OPT}}
\newcommand{\e}{\text{E}}

\newtheorem{definition}{Definition}
\newtheorem{theorem}{Theorem}
\newtheorem{lemma}{Lemma}
\newtheorem{corollary}{Corollary}

\newtheorem*{remark}{Remark}

\SetKwComment{Comment}{/* }{ */}
\SetKwData{KwInput}{\bf{Input: }}
\SetKwData{KwOutput}{\bf{Output: }}

\begin{document}

\maketitle

\begin{abstract}


 We show how to take a regression function $\hat{f}$ that is appropriately ``multicalibrated'' and efficiently post-process it into an approximately error minimizing classifier satisfying a large variety of fairness constraints. The post-processing requires no labeled data, and only a modest amount of unlabeled data and computation.  
  The computational and sample complexity requirements of computing $\hat f$ are comparable to the requirements for solving a single fair learning task optimally, but it can in fact be used to solve {\em many} different downstream fairness-constrained learning problems  efficiently. Our post-processing method  easily handles intersecting groups, generalizing prior work on post-processing regression functions to satisfy fairness constraints that only applied to disjoint groups.  Our work extends recent work showing that multicalibrated regression functions are ``omnipredictors'' (i.e. can be post-processed to optimally solve unconstrained ERM problems) to constrained optimization. 
\end{abstract}

\section{Introduction}
The most common technical framing for fair machine learning is as
constrained optimization. The goal is to solve an empirical risk
minimization problem over some class of models $\cH$, subject to
\emph{fairness constraints}. For example, we might ask to find the
best performing model $h \in \cH$ that equalizes false positive
  rates, false negative rates, raw error rates, or positive classification rates across some collection
of groups $\cG$ \citep{hardt2016equality,dwork2012fairness}). For each of these notions of
fairness, there is a continuum of relaxations to consider: rather
than asking that (e.g.) false positive rates be exactly
equalized across groups, we could ask that they deviate by not more
than 5\%, or 10\%, or 15\%, etc. Because these relaxations trade off
with model accuracy (tracing out Pareto frontiers), it is
common to explore the entire range of tradeoffs for a particular
family of fairness constraints (see
e.g. \citep{agarwal2018reductions,kearns2018preventing}). 

Each of these are distinct problems that seemingly require training fresh models on the data. And each of these problems can be computationally expensive to solve: for example, the ``reductions'' approach of \cite{agarwal2018reductions} requires solving roughly $\log |\cG|/\epsilon^2$ empirical risk minimization problems over $\cH$ to produce an $\epsilon$-approximately optimal solution to any one of them, and the computations cannot be reused. Our goal is to understand when we can pre-compute a single regression model $\hat f$ which is sufficient to solve \emph{all} of the fair machine learning problems described above, each as only a computationally easy \emph{post-processing} of $\hat f$, without sacrificing accuracy. 


\subsection{Our Results in Context}
The idea of \emph{post-processing} a trained model $\hat f$ in order
to satisfy fairness constraints is not new. For example,
\cite{hardt2016equality} propose a simple post-processing of a
regression function $\hat f$ to derive a classifier subject to false
positive or negative rate constraints. However, the conditions under
which such post-processing approaches work are not yet well
understood. In particular, two important questions about post-proceessing
remain: First, \emph{how} should one algorithmically post-process a
regression function $\hat f$ to obtain a good (and fair)
downstream classifier, and \emph{what properties} must $\hat f$
satisfy? Prior work \citep{hardt2016equality} handles
the case in which the groups $\cG$ are disjoint, by finding a
different \emph{thresholding} of $\hat f$ for each group $g \in
\cG$. Is there a simple, efficient post-processing that applies in the
common case that groups intersect --- as is the case e.g. with groups
defined by race and gender? Second and similarly, \citet{hardt2016equality} and
\cite{corbett2017algorithmic} show that this post-processing yields
the Bayes Optimal fair classifier if $\hat f$ is the true conditional
label distribution. Are there weaker conditions on $\hat f$ (that can be
efficiently satisfied from only a polynomial number of samples) that
also lead to guarantees? We answer both of these questions in the
affirmative. 

\paragraph{Post Processing for Intersecting Groups} Suppose we have
$k = |\cG|$ groups that are intersecting (e.g. divisions of a population
by race, gender, income, nationality, etc.) A naive reduction to the
post-processing approach of \cite{hardt2016equality} would consider
all $2^k$ (now disjoint) intersections of groups, and find a separate
thresholding of $\hat f(x)$ for each one. We show that even when
groups intersect, for a variety of fairness constraints, the
optimal post-processing $\hat h$ remains a thresholding that depends on only $k$ parameters
$\lambda_g$, one for each group $g$. The value at which to threshold $\hat f(x)$ now depends only on these $k$ parameters \emph{and the subset of groups that $x$ is contained in}. We give a simple, efficient algorithm to compute these optimal post-processings. The algorithm is efficient in the worst case --- i.e. it does not have to call any heuristic ``learning oracle'' as direct learning approaches do \citep{agarwal2018reductions,kearns2018preventing}, and requires access only to a modest amount of \emph{unlabeled} data from the underlying distribution. 

\paragraph{Accuracy Guarantees from Multicalibration} As in 
\cite{hardt2016equality} when $\hat f$ is the Bayes
optimal regression function, for a variety of fairness constraints,
our post-processing $\hat h$ is the Bayes optimal fair classifier. But
in general we cannot hope to learn the Bayes optimal regression
function $\hat f$ given only a polynomial amount of data and
computation. We show that substantially weaker conditions suffice: If
$\hat f$ is \emph{multicalibrated} with respect to a class of models
$\cH$, a class of groups $\cG$, and a simple class of functions derived from $\cH$ and $\cG$, then the 
post-processing $\hat h$ of $\hat f$ will be as accurate as the best
fair model in $\cH$ while satisfying all of the fairness constraints
defined over $\cG$. Learning a multicalibrated predictor with respect
to these classes can be done with polynomial sample complexity in an
oracle-efficient manner whenever $\cH$ and $\cG$ have polynomial
VC dimension --- and so both the sample and computational complexity of computing $\hat f$ are comparable to what would be required to directly solve a single instance of a fairness constrained optimization problem over $\cH$. 
\paragraph{Experimental Evaluation}
We provide preliminary experimental evaluation of our method on a dataset derived from Pennsylvania Census data provided by the Folktables package \citep{ding2021retiring}. The experimental results support the theoretical findings that our method is able to quickly converge to a solution that approximately satisfies the given target constraints.

\paragraph{} Taken together, our results contribute to the following
conclusion: even when the notion of fairness that is eventually desired
in downstream tasks is one that approximately equalizes some notion of
statistical error across groups, this is \emph{not} necessarily what
should be trained. Aiming instead for group-wise fidelity in the form
of \emph{multicalibration} provides the flexibility to deploy an
optimal downstream model subject to a variety of fairness constraints
without destroying information that would be needed to later relax or
tighten those constraints, to remove them or to add more, or to change
their type.

\subsection{Additional Related Work}

There are a number of other papers that study the problem of converting a regression (or ``score'') function into a classification rule in the context of fair machine learning. For example, \cite{woodworth2017learning} shows that post-proceessing a learned binary classification model to satisfy fairness constraints can be substantially suboptimal even when the hypothesis class under consideration contains the Bayes optimal predictor, which motivates a focus on post-processing regression functions instead. \cite{yang2020fairness} study the structure of the Bayes optimal fair classifier for several notions of fairness when groups are intersecting, under a continuity assumption on the underlying distribution; they do not consider utility guarantees for post-processing a regression function that does not completely represent the underlying probability distribution. 
\cite{wei2021optimized} and \cite{alabdulmohsin2021near} give post-processing algorithms that transforms a score function into a regression function that optimizes different measures of accuracy subject to a variety of fairness constraints using a similar primal/dual perspective that we use in this paper. But these papers do not address the two main questions we raise in our work: intersecting groups, and efficiently learnable conditions on the score function that lead to utility guarantees (they assume that in the limit the true conditional label distribution is learnable and given as input to their algorithm) 

In proving our accuracy bounds, we draw on a recent line of work on
\emph{multicalibration}
\citep{multicalibration,multiaccuracy,momentmulti,dwork2021outcome,onlinemulti}. In
particular, \cite{omnipredictors} showed that regression functions
that are multicalibrated with respect to a class of models $\cH$ are
\emph{omni-predictors} with respect to $\cH$, which means that they
can be post-processed to perform as well as the best model in $\cH$
with respect to any convex loss function satisfying mild technical
conditions. The results in our paper can be viewed as being a 
\emph{constrained optimization} parallel to \cite{omnipredictors},
which studies \emph{unconstrained} optimization.

Several other papers also use multicalibration of intermediate
statistical products to argue for the utility of downstream
models. \cite{decisioncalibration} consider the problem of calibrating
a model to the utility function of a downstream utility maximizing
decision maker to preserve the usefulness of the model for the
decision-maker.  \cite{proxies} show that a \emph{proxy-model} for a
protected attribute can be useful in enforcing fairness constraints on
a downstream model when the real protected attribute is not available
if the proxy is appropriately multicalibrated.  \cite{scaffolding}
propose training a multicalibrated predictor on the level sets of a
low dimensional learned representation as a means of obtaining Bayes
optimality.  \cite{kim2022universal} show that a predictor that is
multicalibrated with respect to a function class can adapt to new
domains with covariate shift as well as a model trained using
propensity-score reweighting via any propensity score function in the
class.

\cite{anon} independently study a similar problem. Our two papers derive a closely related but incomparable set of results. \cite{anon} tackles a more general problem, and studies a richer set of objective functions and constraints (whereas we restrict attention to the classification error objective and fairness motivated constraints). In contrast, in our paper, we are able to take advantage of the additional structure of our problem to derive improved bounds. In particular, we can handle intersecting groups (with running time and sample complexity depending polynomially on the number of groups), whereas \cite{anon} requires taking all of the exponentially many group intersections to recover disjoint groups---which leads to an exponential (in the number of groups) loss in the running time and sample complexity. Similarly, they require more precise multicalibration as more groups are added, whereas we derive results from a multicalibrated predictor with parameter that is independent of the number of groups. 

\subsection{Limitations}
\label{sec:limit}
This work explores approaches to fairness that make a jump between complex and ambiguous social ideas of fairness and mathematical guarantees such as equality of false positive rates between groups of individuals. Our work can be applied only when  evaluating the membership of an individual to a group is well-defined, and when consideration of group membership is legal\footnote{Note that in some contexts such as consumer lending in the United States, direct consideration of membership in protected groups such as race is illegal. However, demographic information can be used when designing and auditing a decision-making process, so long as those characteristics are not part of the real-time lending decisions.}, and when the training data is representative of the underlying population. There will be contexts in which these  assumptions are either false, overly simplistic, or bypass larger questions: e.g. an application might be fair in its performance but still entirely unethical, or groups may be systematically underrepresented in datasets. In the latter case, the guarantees of our work cannot be interpreted as guarantees relative to the optimal predictor for the 
true distribution over groups. 

 It is worth noting that while the assumption that we can define group membership of individuals  simplifies the complexities of personal identity, this work does improve on the existing literature on post-processing approaches to fairness in that it allows for \textit{non-disjoint}, or intersectional, group membership. In general, this work (and all work in algorithmic fairness) should not be assumed to ``solve" fairness. Instead it should be taken as a tool in a larger system to evaluate and remediate issues of fairness and ethics in machine learning. 

\section{Preliminaries}
We study binary classification problems.  Let $\cX$ be an arbitrary feature space and $\cY = \{0,1\}$ be a binary label space. A classification problem is defined by an underlying data distribution $\cD \in \Delta(\cX \times \cY)$. In general we will not have direct access to the data distribution, but rather only to samples drawn i.i.d. from $\cD$. We let $D$ denote a dataset of size $n$, drawn i.i.d. from $\cD$: $D \sim \cD^n$.

We will study both regression functions $f:\cX \rightarrow \mathbb{R}$ and classification functions (classifiers)  $h:\cX\rightarrow \{0,1\}$. In general we will use $f$ and variants ($f^*, \hat{f}$, etc.) when speaking of regression functions and $h$ and variants $(h^*, \hat{h}$, etc.) when speaking of classification functions. Our interest will be in regression functions used to estimate conditoinal label expectations in binary prediction problems, and so the natural range of our regression functions will be (discrete subsets of) $[0,1]$.

\begin{definition}[Bayes Optimal Regression Function]
We let $f^*$ denote the \emph{Bayes Optimal Regression Function} $f^* = \arg\min_{f} \E_{(x,y) \sim \cD}(f(x) - y)^2$ which takes value:
$$f^*(x) = \E_{(x',y') \sim \cD}[y' | x' = x]$$
\end{definition}

\begin{remark}
The property of $f^*$ that we are interested in is that it encodes the true conditional label expectations. The fact that it minimizes squared error is not important --- $f^*$ would also minimize any other proper loss function. 
\end{remark}

Let $\cD_\cX$ denote the marginal distribution on features induced by projecting $\cD$ onto $\cX$. Note that we can equivalently sample a pair $(x,y) \sim \cD$ by first sampling $x \sim \cD_\cX$ and then sampling $y = 1$ with probability $f^*(x)$ and $y = 0$ otherwise. 

Given a classifier $h:\cX\rightarrow \cY$, and a data distribution $\cD$, we can refer to various notions of error. We will be interested in error rates not just overall, but on subsets of the data that we call \emph{groups} (which we might think of as e.g. demographic groups when the data represents people). We will represent groups by group indicator functions:
\begin{definition}
Let $\cG$ denote a collection of \emph{groups}, each represented by a group indicator function $g:\cX\rightarrow \{0,1\}$. If $g(x) = 1$ we say that $x$ is a \emph{member} of group $g$. Let $I$ denote the group containing all elements ($I(x) = 1$ for all $x$). We will always assume that $I \in \cG$. 
\end{definition}

We allow $\cG$ to contain arbitrarily intersecting groups. We can now define error rates over these groups, and a notion of fairness.
\begin{definition}
\label{def:err}
The error of a classifier $h:\cX\rightarrow \cY$ on a group $g$ as measured over  distribution $\cD$ is:
\[\err(h,g,\cD) = \Pr_{(x,y) \sim \cD}[h(x) \neq y | g(x) = 1] = \E_{(x,y) \sim \cD}[\ell(h(x), y) \vert g(x) = 1]\]
The false positive rate of a classifier $h:\cX\rightarrow \cY$ on a group $g$ is:
\[\rho(h,g,\cD) = \Pr_{(x,y) \sim \cD}[h(x) \neq y | y = 0, g(x) = 1]\]
When $h$ is a randomized classifier, the probabilities are computed over the randomness of $h$ as well.
For convenience, we write $\err(h) = \err(h,I,\cD), \rho_g(h) \equiv \rho(h, g, \cD)$, and $\rho(h) \equiv \rho(h, I, \cD)$.
\end{definition}
\begin{definition}
\label{def:fairness}
We say that classifier $h: \cX \rightarrow \cY$ satisfies $\gamma$-False Positive (FP) Fairness with respect to $\cD$ and $\cG$ if for all $g \in \cG$,
$$w_g \left\vert \rho_g(h) - \rho(h) \right\vert \le \gamma.$$
where $w_g = \Pr_{(x,y) \sim \cD} [g(x) = 1, y=0]$.


\end{definition}
\begin{remark}
In the above definition, we include a multiplicative factor that provides slack in the fairness guarantee for groups with small weight over the distribution. This approximation parameter is necessary, as statistical estimation over small groups is inherently more difficult. By including this factor directly into our fairness constraint rather than incorporating it indirectly into  sample complexity guarantees, we are able to elide exposition in our later proofs. An equivalent (up to reparameterization) alternative would be to remove the $w_g$ term in our constraints, but to provide guarantees only for groups for whom $w_g$ is sufficiently large. 
\end{remark}
For the sake of brevity and clarity, in the main body of this paper we prove all results in the context of $\gamma$-False Positive Fairness. We discuss the modifications necessary to extend the results to other fairness notions in Appendix \ref{ap:fairness}.

We will study how to derive classifiers with optimal error properties, subject to fairness-motivated constraints on group-wise error rates from regression functions satisfying \emph{multicalibration} constraints \citep{multicalibration}. Informally, if $\hat f$ is multicalibrated with respect to a class of functions $C$, then $\hat f(x)$ takes values equal to $f^*(x)$ in expectation, even conditional on both the value of $\hat f(x)$ and on the value of $c(x)$ for each $c \in C$. We use two variants. The first (multicalibration in expectation) was defined and studied in \citep{omnipredictors}:

\begin{definition}[Multicalibration in Expectation \citep{multicalibration,omnipredictors}]
Fix a distribution $\cD$ and let $C$ be a collection of functions $c:\cX\rightarrow \{0,1\}$. We say that a predictor $\hat f:\cX\rightarrow R$ where $R$ is some discrete domain $R \subseteq [0,1]$ is $\alpha$-approximately multicalibrated with respect to $C$ if for every $c\in C$:
\begin{align*}
    &\sum_{v \in R}\Pr_{x \sim \cD_\cX}[\hat f(x) = v]\cdot \left|\E_{(x,y)}\left[(\hat f(x)-f^*(x))\cdot c(x) \middle| \hat f(x) = v \right] \right|
    \\
    &= \sum_{v \in R}\Pr_{x \sim \cD_\cX}[\hat f(x) = v] \cdot \frac{\Pr_{x \sim \cD_\cX}[c(x)=1|\hat f(x) = v]}{\Pr_{x \sim \cD_\cX}[c(x)=1|\hat f(x) = v]}\cdot \left|\E_{(x,y)}\left[(\hat f(x)-f^*(x))\cdot c(x) \middle| \hat f(x) = v \right] \right| \\
    &= \sum_{v \in R}\Pr_{x \sim \cD_\cX}[c(x)=1, \hat f(x) = v] \cdot \left|  \E_{(x,y)}\left[\hat{f}(x) - f^*(x) \middle| \hat f(x) = v, c(x) = 1 \right] \right| \\
    &= \sum_{v \in R}\Pr_{x \sim \cD_\cX}[c(x)=1, \hat f(x) = v] \cdot \left|v - \E_{(x,y)}\left[f^*(x) \middle| \hat f(x) = v, c(x) = 1 \right] \right| \\
    &\leq \alpha
\end{align*}
\end{definition}

 We will require this notion of multicalibration with respect to the set of groups $\cG$ with which we define our fairness constraints, for the classifiers $h \in \cH$, and for the intersection of these classes $\cG \times \cH = \{g(x) \cdot h(x) \vert g \in \cG, h \in \cH\}$. 
We will also need a variant of multicalibration that is tailored to two-argument functions $c:\cX\times R \rightarrow \{0,1\}$ in order to argue about the properties of thresholding functions, which take both a value $x \in \cX$ and a threshold in a discrete domain $R \subseteq [0,1]$, and which threshold predictions to $\{0,1\}$.

In this definition, when we condition on $\hat f(x) = v$, we also
condition on the second argument of $c$ taking the same value $v$. We
call this \emph{joint}-multicalibration. It is only a modest  generalization of multicalibration: we verify in Appendix \ref{ap:jointmultical} that existing algorithms for obtaining multicalibrated predictors easily extend to our definition of joint
multicalibration. 

\begin{definition}[Joint Multicalibration in Expectation]
  We say that a predictor $\hat f:\cX\rightarrow R$ where $R$ is some
  discrete domain $R \subseteq [0,1]$ is $\alpha$-approximately
  jointly multicalibrated with respect to a class of functions
  $c: \cX \times R \to \{0,1\}$ if for every $c \in C$:
\[
\sum_{v \in R}\Pr[\hat f(x) = v, c(x,v)=1]\cdot \left|\E_{(x,y)}\left[(\hat f(x)-f^*(x)) \middle| \hat f(x) = v, c(x,v)=1 \right] \right| \leq \alpha,
\]
which is equivalent to
\[
\sum_{v \in R}\Pr[\hat f(x) = v]\cdot \left|\E_{(x,y)}\left[c(x,v) \cdot (\hat{f}(x)-f^*(x)) \middle| \hat f(x) = v\right] \right| \leq \alpha.
\]
\end{definition}

\section{The Structure of an Optimal Post-Processing}
In this section, we consider a fairness-constrained optimization problem  that seeks to find the (distribution over) model(s) in $\cH$ that minimize error subject to a constraint on group-wise false positive rates: 
\begin{align}
    &\min_{h \in \Delta\cH} && \err(h)  \label{eq:fairnessLP} \\ 
    &\textrm{s.t. for each } g \in \cG: &&
    w_g | \rho_g(h) - \rho(h)| \leq \gamma, \nonumber 
\end{align}
where $w_g, \rho_g(h),$ and $\rho(h)$ are defined as in Definition \ref{def:err}.

It will be useful for us to re-write this optimization problem in terms of the conditional label expectation $f^*(x)$. Since later in the paper we will want to replace $f^*(x)$ with a different regression function $f$ that is easier to learn, we define the linear program generically in terms of an arbitrary regression function $f$:

\begin{definition}
\label{def:lp}
Let $f: \cX \rightarrow R \subseteq [0,1]$ be some regression function and let $\gamma \in \mathbb{R}_+$. Define $\psi(f,\gamma,\cH)$ to be the following optimization problem:
\begin{align*} 
    &\min_{h \in \Delta \cH} &&\E_{x \sim \cD_\cX} [f(x)\ell(h(x), 1) + (1-f(x))\ell(h(x),0)] \\
    &\textrm{s.t. for each } g \in \cG: && \left\vert \Ex[\ell(h(x),0)g(x)(1-f(x))] - \beta_g \Ex\left[\ell(h(x),0)(1-f(x))\right] \right\vert \le \gamma,
\end{align*}
where $\beta_g = \Pr[g(x)=1 \vert y=0]$.
\end{definition}
\begin{restatable}{lemma}{optimalfair}
\label{lem:optimal-fair-lp}
Let $f^*$ be the Bayes optimal regression function over $\cD$. Then optimization problem $\psi(f^*, \gamma,\cH)$ is equivalent to the fairness-constrained optimization problem \eqref{eq:fairnessLP}. 
\end{restatable}
\ifarxiv
\else
The proof is in Appendix \ref{ap:proofs}.
\fi 

We will be interested in the structure and properties of the optimal solution to $\psi(f, \gamma,\cH)$, which will be elucidated via its  Lagrangian. Note that the optimization problem has $2|\cG|$ linear inequality constraints. 
 Let $\lambda = \{\lambda_g^{\pm}\}_{g \in \cG}$ denote the vector of $2|\cG|$ dual variables corresponding to those constraints, and write $\lambda_g = \lambda_g^+ - \lambda_g^-$. 
 
\begin{definition}[Lagrangian]
\label{def:lag}
Given any regression function $f$, we define a Lagrangian of the optimization problem  $\psi(f, \gamma,\cH)$ as $L_f:\cH\times \mathbb{R}^{2|\cG|}\rightarrow \mathbb{R}$: 
\begin{align*}
   L_{f}(h,\lambda)  &=
\E_{x \sim \cD_\cX} \bigg[
f(x)\ell(h(x), 1) + (1-f(x))\ell(h(x),0) \\ 
& \quad + \sum_{g \in \cG}\lambda_g^+ \big( \ell(h(x),0)g(x)(1-f(x)) - \beta_g \ell(h(x),0)(1-f(x)) - \gamma \big) \\
& \quad + \sum_{g \in \cG} \lambda_g^-\big(\beta_g \ell(h(x),0)(1-f(x)) - \ell(h(x),0)g(x)(1-f(x)) - \gamma \big) \bigg]
\end{align*}

For convenience, given a Bayes optimal regressor $f^*$, we write $L^* = L_{f^*} $. Similarly, given some other regressor $\hat{f}$, we write $\hat{L} = L_{\hat{f}}$.
\end{definition} 

Let $\cH_A = 2^{\cX}$ be the set of all Boolean functions $f:\cX \rightarrow \{0,1\}$. We will consider solving our optimization problem over this set of functions $\cH_A$.
\begin{definition}[Optimal post-processed classifier] 
We say that a classifier $h_f$ is an optimal post-processing of $f$ if there exists a vector $\lambda^f$ such that the following primal/dual optimality conditions are simultaneously met:

\[
    h_{f}(x) \in \arg\min_{h \in \cH_A} L_f(h, \lambda^f) \quad 
    \lambda^f \in \arg\max_{\lambda \in \mathbb{R}^{2|\cG|}} L_f(h_f, \lambda).
\]
For convenience, we write 
\begin{align*}
    h^*(x) = h_{f^*}(x)\quad&\text{and}\quad\lambda^* = \lambda^{f^*}\\
    \hat{h}(x) = h_{\hat{f}}(x)\quad&\text{and}\quad\hat{\lambda} = \lambda^{\hat{f}}
\end{align*}
where $f^*$ is the Bayes optimal regressor and $\hat{f}$ is any other regressor. We will write $\lambda_g^*$ and $\hat \lambda_g$ to refer to the dual variable in $\lambda^*$ and $\hat \lambda$ for group $g$, respectively. We observe that as the optimal solution to the Lagrangian minimax optimization problem, $h^*(x)$ is the Bayes optimal classifier subject to the fairness constraints in \ref{eq:fairnessLP}.
\end{definition}

\begin{restatable}{lemma}{lemh}
\label{lem:h}
The optimal post-processed classifier $h$ of $\psi(f, \gamma, \cH_A$) for some regressor $f$ takes the following form:
\[
h(x) = \begin{cases}
1, & \text{if } f(x) > \frac{1 + \sum_{g \in \cG} \lambda_g (g(x)-\beta_g)}{2 + \sum_{g \in \cG} \lambda_g (g(x)-\beta_g)} \text{ and } 2 + \sum_{g \in \cG} \lambda_g (g(x)-\beta_g) > 0,\\
0, & \text{if } f(x) < \frac{1 + \sum_{g \in \cG} \lambda_g (g(x)-\beta_g)}{2 + \sum_{g \in \cG} \lambda_g (g(x)-\beta_g)} \text{ and } 2 + \sum_{g \in \cG} \lambda_g (g(x)-\beta_g) > 0,\\
1, & \text{if } f(x) < \frac{1 + \sum_{g \in \cG} \lambda_g (g(x)-\beta_g)}{2 + \sum_{g \in \cG} \lambda_g (g(x)-\beta_g)} \text{ and } 2 + \sum_{g \in \cG} \lambda_g (g(x)-\beta_g) < 0,\\
0, & \text{if } f(x) > \frac{1 + \sum_{g \in \cG} \lambda_g (g(x)-\beta_g)}{2 + \sum_{g \in \cG} \lambda_g (g(x)-\beta_g)} \text{ and } 2 + \sum_{g \in \cG} \lambda_g (g(x)-\beta_g) < 0.
\end{cases}
\]

In the edge case in which $f(x) = \frac{1 + \sum_{g \in \cG} \lambda_g (g(x)-\beta_g)}{2 + \sum_{g \in \cG} \lambda_g (g(x)-\beta_g)}$, $h(x)$ could take either value and might be randomized.
\end{restatable}
\ifarxiv
\else
The proof is in Appendix \ref{ap:proofs}.
\fi 

\subsection{Computing the optimally post-processed classifier}
To approximate $h$ given $f$, we need to compute an approximately optimal solution to the linear program $\psi(f, \gamma,\cH_A)$. We can do this by playing a no-regret algorithm over the dual variables $\lambda$ and best response over the primal variables as defined in Definition \ref{def:BR_h}. We can approximate the losses to the no-regret algorithm from a finite sample of unlabelled data of size scaling logarithmically in the number of constraints and linearly in the number of rounds $T$ of the no regret dynamics, which using standard techniques we can show yields an approximately optimal solution to the original LP. The  algorithm is described in Algorithm \ref{alg:descent}. We state its approximate guarantees then spend the rest of this section formalizing the structure necessary for the result.


\begin{restatable}{theorem}{algmain} \label{thm:alg-main}
Let $\opt$ be the objective value of the optimal solution to $\psi(f, \gamma, \cH_A)$. Then, for any $C \in \mathbb{R}$, after $T = \frac{1}{4} \cdot C^2 \cdot \left(C^2 + 4|\cG| \right)^2 $ iterations, Algorithm \ref{alg:descent} outputs a randomized hypothesis $\bar{h}$ such that $\text{err}(\bar{h}) \leq \opt + \frac{2}{C}$ and $w_g | \rho_g(\bar{h})- \rho(\bar{h})| \leq \gamma + \frac{1}{C} + \frac{2}{C^2}$.
\end{restatable} 

In order to prove Theorem \ref{thm:alg-main}\ifarxiv\else{ (which is proved fully in Appendix \ref{ap:proofs}),}\fi we first must specify the game formulation of the problem and demonstrate that constraining the dual player still allows for an adequate approximation to the original problem.

\paragraph{Game formulation}

We pose the optimization of our original linear program as a zero-sum game between a primal (minimization) player who plays over the set of hypotheses and a dual (maximization) player who plays over the set of dual variables. The utility function of the game is the Lagrangian of our linear program as stated in Definition \ref{def:lag}. The value of this game is given by
\begin{align*}
        \min_{h \in \Delta \cH} \max_{\lambda \in \mathbb{R}^{2|\cG|}} L_f(h, \lambda).
\end{align*}

\paragraph{Constraining the linear program}
In order to compute an approximate minimax solution to this game, we need to constrain the strategy space of the dual player. 
\ifarxiv
\else
That is, we need to bound the dual space to a region $\Lambda = \left\{ \lambda \in \mathbb{R}^{2\mathcal{G}}  \big\vert \Vert \lambda \Vert_1 \le C \right\}$. We call this constrained version of the problem the $\Lambda$-bounded Lagrangian problem, which has value 
\begin{equation} \label{eq:game}
    \min_{h \in \Delta \cH} \max_{\lambda: |\lambda|_1 \leq C} L_f(h, \lambda).
\end{equation}
\fi
We can apply the minimax theorem to this bounded game to see: 
\begin{align*}
    \min_{h \in \Delta \cH} \max_{\lambda: |\lambda|_1 \leq C} L_f(h, \lambda) \equiv \max_{\lambda: |\lambda|_1 \leq C} \min_{h \in \Delta \cH} L_f(h, \lambda).
\end{align*}

We will only be able to achieve an approximate solution to the problem, which we define as follows.
\ifarxiv
\else
\begin{definition}
We say that $(h, \lambda)$ is a $v$-approximate minimax solution to the $\Lambda$-bounded Lagrangian problem $L_f$ if 
$ L_f(h, \lambda) \le \min_{h' \in \Delta \mathcal{H}} L_f(h', \lambda) + v$
and 
$ L_f(h, \lambda) \ge \max_{\lambda' \in \Lambda} L_f(h, \lambda') - v.$
\end{definition}
\fi 
An approximate minimax solution to this bounded version of the problem is also an approximate solution to the original problem we described in Equation \ref{eq:fairnessLP}:
\ifarxiv
\else
\begin{restatable}{theorem}{approxminmax}[\cite{kearns2018preventing}]
Let $(h, \lambda)$ be a $v$-approximate minimax solution to the $\Lambda$-bounded Lagrangian problem $L_f$ and let $\mathrm{OPT}$ be the objective value of the optimal solution to $\psi(f, \gamma, \cH_A)$. Then, $err(h) \leq \text{OPT} + 2v$, and $\forall g \in \mathcal{G}, w_g|\rho_g(h) - \rho(h)| \le \gamma + (1+2v)/C.$
\label{thm:approxminmax}
\end{restatable} 
\fi 
\paragraph{Approximate equilibrium of the constrained game} Now, we can proceed with no-regret play to find an approximate solution to the game. The dual player will play projected gradient descent over their vector $\lambda$ and the primal player will best respond, as described in Algorithm ~\ref{alg:descent}.

\begin{restatable}{theorem}{regret} 
\label{thm:regret}
Algorithm ~\ref{alg:descent} returns an $\epsilon-$approximate equilibrium solution to the zero-sum game defined by Equation \ref{eq:game} after $T = \frac{1}{4\epsilon^2}\left( \frac{1}{\epsilon^2} + 4|\cG| \right)^2 $ rounds. 
\end{restatable}
\ifarxiv
\else 
The proof of Theorem \ref{thm:regret} is in Appendix \ref{ap:proofs}. Combining Theorem \ref{thm:approxminmax} and Theorem \ref{thm:regret} gives us the proof of Theorem \ref{thm:alg-main}, which appears in Appendix \ref{ap:proofs}.
\fi 

\begin{algorithm}\label{alg:descent}
\caption{Projected Gradient Descent Algorithm}
    \KwInput{$D$: dataset, $f: \cX \to [0, 1]$: regression function, $\cG$: groups, $\gamma$: tolerance on fairness violation, $C$: bound on dual $(\|\lambda\|_1 \leq C)$, $\eta$: learning rate}
    
    Initialize dual vector $\lambda^0 = {\bf 0}$ and set $ T = \frac{1}{4} \cdot C^2 \cdot (C^2 + 4|\cG|)^2 $.
    
    \For{$t = 1, \ldots, T$}
    {
    Primal player updates $h_t$
    \[
        h_t(x) = \begin{cases}
        1, & \text{if }  f(x) \geq \frac{1 + \sum_{g \in \cG} \lambda^{t-1}_g (g(x)-\beta_g)}{2 + \sum_{g \in \cG} \lambda^{t-1}_g (g(x)-\beta_g)} \text{ and } 2 + \sum_{g \in \cG} \lambda^{t-1}_g (g(x)-\beta_g) > 0,\\
        0, & \text{if } f(x) < \frac{1 + \sum_{g \in \cG} \lambda^{t-1}_g (g(x)-\beta_g)}{2 + \sum_{g \in \cG} \lambda^{t-1}_g (g(x)-\beta_g)} \text{ and } 2 + \sum_{g \in \cG} \lambda^{t-1}_g (g(x)-\beta_g) > 0,\\
        1, & \text{if } f(x) \leq \frac{1 + \sum_{g \in \cG} \lambda^{t-1}_g (g(x)-\beta_g)}{2 + \sum_{g \in \cG} \lambda^{t-1}_g (g(x)-\beta_g)} \text{ and } 2 + \sum_{g \in \cG} \lambda^{t-1}_g (g(x)-\beta_g) < 0,\\
        0, & \text{if } f(x) > \frac{1 + \sum_{g \in \cG} \lambda^{t-1}_g (g(x)-\beta_g)}{2 + \sum_{g \in \cG} \lambda^{t-1}_g (g(x)-\beta_g)} \text{ and } 2 + \sum_{g \in \cG} \lambda^{t-1}_g (g(x)-\beta_g) < 0, \\
        1, & \text{if } 2 + \sum_{g \in \cG} \lambda^{t-1}_g (g(x)-\beta_g) = 0
        \end{cases}
    \]
    Compute
    \begin{align*}
        \hat \rho^t_g& = \E_{(x, y) \sim D} [\ell(h_t(x),0)g(x)(1- f(x))] \text{ for all } g \in \cG,\\
        \hat \rho^t &= \E_{(x, y) \sim D} [\beta_g \ell(h_t(x),0)(1-  f(x))], \text{ where } \beta_g = \Pr[g(x) = 1 | y = 0]
    \end{align*}
    Dual player updates 
    \begin{align*}
    \lambda_{g}^{t, +} &= \max(0, \lambda_g^{t,+} + \eta \cdot (\hat \rho^t_g - \hat \rho^t - \gamma)), \\
    \lambda_{g}^{t, -} &= \max(0, \lambda_g^{t,-} + \eta \cdot (\hat \rho^t - \hat \rho^t_g - \gamma)).
    \end{align*}
    
    Dual player sets $\lambda^t =\sum_{g \in \cG} \lambda_g^{t, +} - \lambda_g^{t, -}$.
    
    If $\|\lambda^t\|_1 > C$, set $\lambda^t = \arg \min _{ \{ \tilde{\lambda} \in \mathbb{R}^{2\cG} \vert \Vert\tilde{\lambda}\Vert_1 \le C \}} \Vert \lambda_t - \tilde{\lambda} \Vert_2^2$.
    }
    \KwOutput $\bar{h} := \frac{1}{T} \sum_{t=1}^T \hat{h}_t$, a uniformly random classifier over all rounds' hypotheses.
\end{algorithm}

\subsection{Beginning with a Multicalibrated Regression Function $\hat{f}$} \label{sec:multicalibrated-regressor}

Thus far, we have considered the optimization problem $\psi(f, \gamma,\cH_A)$ in the abstract, have characterized its optimal solution $h$, and have given a simple algorithm to find $\bar h$, an approximately optimal solution. When $f = f^*$, $h = h^*$ is the Bayes optimal fair classifier, and $\bar h$ is approximately Bayes optimal.  But in practice, we will not have access to $f^*$, but will instead
only have some surrogate function, which we will call $\hat f (x)$. We will argue that if $\hat f$ is appropriately \textit{multicalibrated}, then it is good enough for our purposes. We will compare the approximate solution $\bar h$ produced by  Algorithm \ref{alg:descent} to the
optimization problem $\psi(\hat f, \gamma, \cH_A)$ which has corresponding Lagrangian $\hat{L}(\hat h, \hat \lambda)$, as defined in Definition \ref{def:lag} to the optimal solution $(h^*, \lambda^*)$ to the optimization problem $\psi(f^*, \gamma, \cH)$ for some constrained class $\cH$, and show conditions under which they are close. 

In order to proceed, we first need to determine what our surrogate function ought to be multicalibrated with respect to. In addition to being $\alpha$-approximately multicalibrated in expectation with respect to $\cG$ and $\cH$, we will require that $\hat f$ be $\alpha$-approximately multicalibrated with respect to $\cG \times \cH = \{g(x) \cdot h(x) \vert g \in \cG, h \in \cH\}$. Furthermore, we will need to require that $\hat f$ be $\alpha$-approximately \textit{jointly} multicalibrated in expectation with respect to a set of thresholding functions, defined below:

\begin{definition}[Set of thresholding functions $\cB$]\label{def:threshold}
Let $x_\cG\in \{0,1\}^{|\cG|}$ denote the group membership indicator vector 
of some point $x$. Define the function 
\[d(v) := \frac{2v-1}{1-v}.\]
Then, let for any $\lambda, x, \beta$
\[s_{\lambda}(x,v) := \mathbbm{1}[\langle \lambda, x_\cG - \beta \rangle \ge d(v)] .\]
Define $\cB = \{s_{\lambda} | \lambda \in \Lambda(C), \beta = \beta_{g_1}, \ldots, \beta_{g_{\vert \cG \vert}}\}$, where $\Lambda(C) = \left\{ \lambda \in \mathbb{R}^{2\mathcal{G}}  \big\vert \Vert \lambda \Vert_1 \le C \right\}$, as defined in Equation \ref{eq:game}  and $\beta_g = \Pr_{(x, y) \sim \cD} [ g(x) = 1 | y = 0]$, as defined in Definition \ref{def:lp}.
\end{definition}
\begin{remark}
When the groups of interest are disjoint, joint multicalibraiton with respect to this class $\cB$ is implied by multicalibration with respect to $\cG$. But when the groups can intersect,  this is not an implication, and asking for joint multicalibration with respect to $\cB$ adds new constraints on $\hat f$.
\end{remark}

Informally, these functions take an example, and map it to a vector of its group membership, indicating whether a $\lambda$-weighting of the example's group membership  is larger than some threshold $d(v)$. We will need joint multicalibration with respect to such functions in order to relate the estimated error of $\hat h$ to its true error. These thresholding functions $\cB$ have a natural relationship to the deterministic thresholded models $h_t$ that we compute at each round of Algorithm \ref{alg:descent}:

\begin{restatable}{lemma}{fixh}
\label{lem:fixh}
Let $h_t$ be the response to $\lambda^{t-1}$ described in Algorithm \ref{alg:descent} at some round $t \in [T]$. Then,
\[
h_t(x) = s_{\lambda^{t-1}}(x, f(x)).
\]
\end{restatable}
\ifarxiv
\else The proof is in Appendix \ref{ap:proofs}.  \fi We verify in Appendix \ref{ap:jointmultical} that a
variant of the multicalibration algorithms given in \cite{multicalibration,omnipredictors} can guarantee joint
multicalibration with respect to $\cB$ as well.

With these preliminaries behind us, we can now state our main  theorem, which shows that for any class of models $\cH$ and class of groups $\cG$, given an appropriately multicalibrated $\hat f$ (with multicalibration requirements depending on $\cH$ and $\cG$), the model $\bar h$ output by Algorithm \ref{alg:descent} achieves an error rate and fairness guarantees comparable to the optimal solution to $\psi(f^*, \gamma, \cH)$:
\ifarxiv
\else
\begin{restatable}{theorem}{finalerror}
\label{thm:final-error}
Set  $C = \sqrt{1/\alpha}$. Let $\hat{f}$ be $\alpha$-approximately multicalibrated in expectation with respect to $\cG$, $\cH$, and $\cG \times \cH$
and $\alpha$-approximately jointly multicalibrated in expectation with respect to $\cB$. Let $\bar{h}$ be the result of running Algorithm \ref{alg:descent} with input $\hat f$ and $C$.
Then, $\err(\bar h) \leq \err(h^*) + \alpha(5 + 2\sqrt{1/\alpha})+2\sqrt{\alpha},$
and for all $g\in \cG, w_g \left \vert \rho_g(\bar h) - \rho(\bar h) \right\vert \leq  w_g \left \vert \rho_g( h^*) - \rho( h^*) \right\vert + w_g \alpha.$
\end{restatable}
\fi
\noindent \textit{Proof Sketch:}
Generalizing notation from the previous sections, let $\err(h) = \E_{x \sim \cD_\cX} [f^*(x)\ell(h(x), 1) + (1-f^*(x))\ell(h(x),0)]$
denote the true error of $h$ on the distribution (i.e. as measured according to the true conditional label distribution $f^*$), and let $\widehat \err(h) = \E_{x \sim \cD_\cX} [\hat f(x)\ell(h(x), 1) + (1-\hat f(x))\ell(h(x),0)]$ denote the error of $h$ as estimated using the surrogate function $\hat f$. At a high level, the proof of Theorem \ref{thm:final-error} will proceed as follows:
\begin{align}
\label{eq:first} \err(h^*) &= L^*(h^*,\lambda^*) \\
\label{eq:second} &\geq L^*(h^*, \hat \lambda)  \\
\label{eq:third} &\approx \hat L(h^*, \hat \lambda) \\
\label{eq:fourth} &\geq \hat L(\hat h, \hat \lambda) \\
\label{eq:fifth} &= \widehat \err(\hat h) \\
\label{eq:barh} &\approx \widehat \err(\bar h) \\
\label{eq:sixth} &\approx \err(\bar h) .
\end{align}
Each of these steps takes a lemma (presented in full in the appendix) to justify, but the logic is at a high level as follows: 
The equalities on lines \ref{eq:first} and \ref{eq:fifth} follow from complimentary slackness: at the optimal solution $(h^*, \lambda^*)$ it must be that for each constraint $g$ either the constraint is exactly tight so that its ``violation" term in the Lagrangian evaluates to 0, or its corresponding dual variable $\lambda_g^\pm = 0$. Thus, all terms in the Lagrangian other than the objective evaluate to 0. The inequality in line \ref{eq:second} follows from the dual optimality condition that $\lambda^* \in \arg \max_\lambda L^*(h^*, \lambda)$ and similarly the inequality in line \ref{eq:fourth} follows from the primal optimality condition that $\hat{h} \in \arg \min_{h\in \cH_A} \hat L(h, \hat \lambda)$. Line \ref{eq:barh} follows from the fact that $\bar h$ is an approximately optimal solution to $\psi(\hat f, \gamma,\cH_A)$. Steps \ref{eq:third} and \ref{eq:sixth} follow from our multicalibration guarantees, the former from multicalibration with respect to groups and our hypothesis class, and the latter from joint multicalibration with respect to the set of thresholding functions from Definition~\ref{def:threshold}. \ifarxiv\else The complete proof is found in Appendix \ref{ap:proofs}. \fi

\section{Experiments}
\label{sec:experiments}
In this section, we evaluate our post-processing algorithm on a dataset derived from Pennsylvania Census data provided by the Folktables package \citep{ding2021retiring}, which we use under its MIT license. The sensitive attributes we use from the dataset are binarized gender and the re-coded detailed race code (RAC1P) to create two classes of overlapping groups. We run our algorithm on top of a regression function $\hat f$ trained using the sklearn gradient-boosted decision trees package --- notably it is \emph{not} guaranteed to be multicalibrated in any of the ways our theorems require! Nevertheless our experiments bear out that our post-processing method performs well even on top of off-the-shelf regression methodologies. 
We expand on our experimental investigation in Appendix \ref{ap:experiments}.

\begin{figure}
        \includegraphics[width=0.5\linewidth]{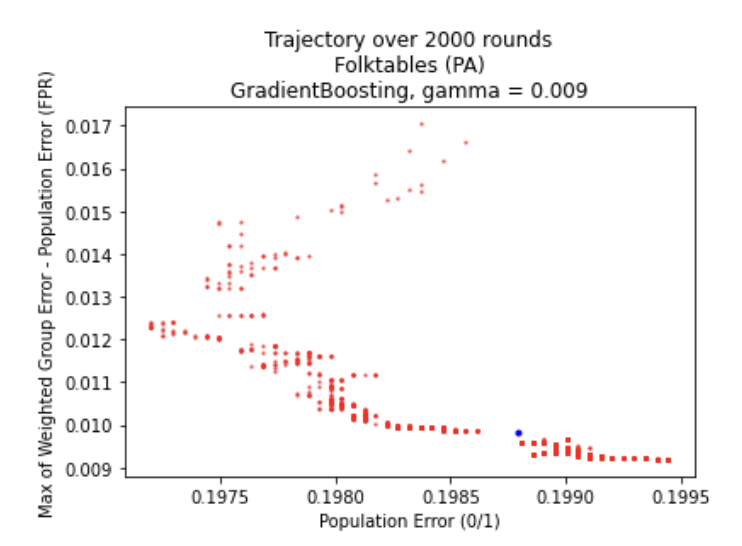}
        \includegraphics[width=0.5\linewidth]{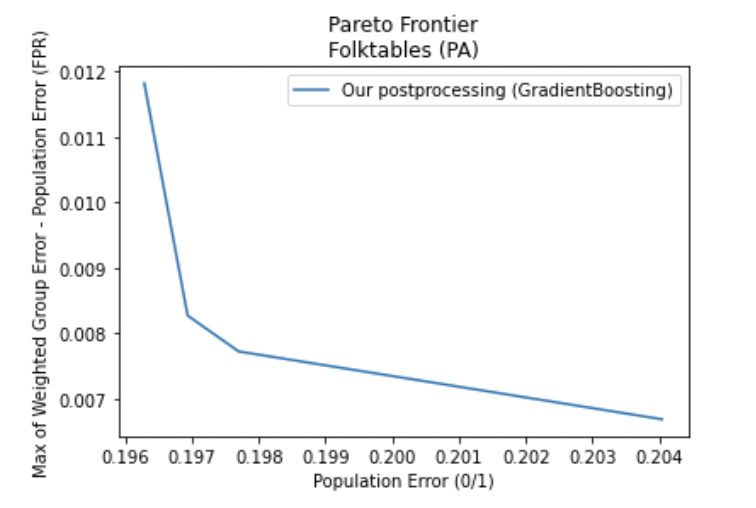}
        \caption{The plot on the left is a trajectory over 2000 iterations of gradient descent of our method post-processing a base model of gradient-boosted regression trees, for a single value of $\gamma = 0.01$. The trajectory starts at the top of the figure and moves downwards with time, and the blue point represents the uniform distribution over the constituent models of the 2000 iterations. The blue point, whose error is  0.1987 and maximum violation of group error - population error is 0.0098, shows our method \iffalse approximately\fi limits the maximum violation to $\gamma$. \iffalse The violation above $\gamma$ is the result of discretization and our approximation guarantees.\fi The plot on the right shows the pareto curve for our method for constraint values ranging between $0.003 \leq \gamma \leq  0.01$, showing that large reductions in false positive rate disparities  cost modestly in error.}
        \label{fig:folk}
\end{figure}

The experimental findings we present support the theoretical analysis that the algorithm quickly converges to classifier approximately satisfying our fairness constraints. 

We emphasize that our post-processing method is extremely lightweight. As a primal/dual algorithm, it is very similar in structure to the ``fair reductions'' method of \citep{agarwal2018reductions}. However where \citep{agarwal2018reductions} needs to solve an ERM problem at every iteration and then evaluate the performance of the resulting trained model, we entirely skip the ERM step and need only evaluate the performance of a thresholded classifier which we have in closed form. 

\subsection*{Acknowledgements}
This work was supported in part by NSF grants AF-1763307, CCF-2217062, and FAI-2147212 and a grant from the Simons Foundation. 

\bibliographystyle{plainnat}
\bibliography{refs}

\newpage 
\appendix
\section{Generalization to other fairness notions}
\label{ap:fairness}

\subsection{False Negative (FN) Fairness}

\begin{definition}
\label{def:FN}
The false negative rate of a classifier $h:\cX\rightarrow \cY$ on a group $g$ is:
\[\rho_{\fn}(h,g,\cD) = \Pr_{(x,y) \sim \cD}[h(x) \neq y | y = 1, g(x) = 1]\]
When $h$ is a randomized classifier, the probabilities are computed over the randomness of $h$ as well.
$\rho_g^{\fn}(h) \equiv \rho_{\fn}(h, g, \cD)$, and $\rho_{\fn}(h) \equiv \rho(h, I, \cD)$.
\end{definition}
\begin{definition}
\label{def:FNfairness}
We say that classifier $h: \cX \rightarrow \cY$ satisfies $\gamma$-False Negative (FN) Fairness with respect to $\cD$ and $\cG$ if for all $g \in \cG$,
$$w_g^\fn \left\vert \rho_g^{\fn}(h) - \rho_{\fn}(h) \right\vert \le \gamma.$$
where $w_g^{\fn} = \Pr_{(x,y) \sim \cD} [g(x) = 1, y=1]$.
\end{definition}

We consider the following fairness-constrained optimization problem:

\begin{align}
    &\min_{h \in \Delta\cH} && \err(h)  \label{eq:fairnessLP-FN} \\ 
    &\textrm{s.t. for each } g \in \cG: &&
    w_g^{\fn} | \rho_g^{\fn}(h) - \rho_{\fn}(h)| \leq \gamma, \nonumber 
\end{align}

\begin{definition}
\label{def:lp-FN}
Let $f: \cX \rightarrow R \subseteq [0,1]$ be some regression function and let $\gamma \in \mathbb{R}_+$. Define $\psi_\fn(f,\gamma,\cH)$ to be the following optimization problem:
\begin{align*} 
    &\min_{h \in \Delta \cH} &&\E_{x \sim \cD_\cX} [f(x)\ell(h(x), 1) + (1-f(x))\ell(h(x),0)] \\
    &\textrm{s.t. for each } g \in \cG: && \left\vert \Ex[\ell(h(x),1)g(x)f(x)] - \beta_g^\fn \Ex\left[\ell(h(x),1)f(x)\right] \right\vert \le \gamma,
\end{align*}
where $\beta_g^\fn = \Pr[g(x)=1 \vert y=1]$.
\end{definition}

\begin{lemma}
\label{lem:optimal-fair-lp-FN}
Let $f^*$ be the Bayes optimal regression function over $\cD$. Then optimization problem $\psi_\fn(f^*, \gamma,\cH)$ is equivalent to the fairness-constrained optimization problem \ref{eq:fairnessLP-FN}. 
\end{lemma}

\begin{proof}
Note that the objective function is equivalent to that of Equation \ref{eq:fairnessLP}, and hence proof of the objectives being equivalent is identical to that of Lemma \ref{lem:optimal-fair-lp}.
For the constraints, note that 
\begin{align*}
    w_g^\fn \vert \rho_g^\fn(h) - \rho_\fn(h) \vert &= \Pr[g(x)=1, y=1] \left\vert \Pr[h(x) = 0 \vert g(x)=1, y=1] - \Pr[h(x)=0 \vert y=1] \right\vert \\
    &= \Pr[g(x)=1, y=1] \bigg\vert \frac{\Pr[h(x)=0, g(x)=1, y=1]}{\Pr[g(x)=1, y=1]} - \frac{\Pr[h(x)=0, y=1]}{\Pr[Y=1]}\bigg\vert\\
    &= \left\vert \Pr[h(x)=0, g(x)=1, y=1] - \frac{\Pr[g(x)=1, y=1] \Pr[h(x)=0, y=1]}{\Pr[Y=1]}  \right\vert \\
    &= \left\vert \Ex[\ell(h(x),1)g(x)f^*(x)] - \frac{\Pr[g(x)=1, y=1]}{\Pr[Y=1]} \Ex\left[\ell(h(x),1)f^*(x)\right] \right\vert \\
     &= \left\vert \Ex[\ell(h(x),1)g(x)f^*(x)] - \Pr[g(x)=1 \vert Y=0] \Ex\left[\ell(h(x),1)f^*(x)\right] \right\vert \\
     &= \left\vert \Ex[\ell(h(x),1)g(x)f^*(x)] - \beta_g^\fn \Ex\left[\ell(h(x),1)f^*(x)\right] \right\vert.
\end{align*}
The result follows. 
\end{proof}

\begin{definition}[Lagrangian]
\label{def:lag-FN}
Given any regression function $f$, we define a Lagrangian of the optimization problem  $\psi_\fn(f, \gamma,\cH)$ as $L_f^\fn:\cH\times \mathbb{R}^{2|\cG|}\rightarrow \mathbb{R}$: 
\begin{align*}
   L_{f}^\fn(h,\lambda)  &=
\E_{x \sim \cD_\cX} \bigg[
f(x)\ell(h(x), 1) + (1-f(x))\ell(h(x),0) \\ 
& \quad + \sum_{g \in \cG}\lambda_g^+ \big( \ell(h(x),1)g(x)f(x) - \beta_g \ell(h(x),1)f(x) - \gamma \big) \\
& \quad + \sum_{g \in \cG} \lambda_g^-\big(\beta_g \ell(h(x),1)f(x) - \ell(h(x),1)g(x)f(x) - \gamma \big) \bigg]
\end{align*}
\end{definition} 

\begin{lemma}
\label{lem:expandedLagFN}
\begin{align*}
L_{f}^\fn(h,\lambda) &= \Ex_{x \sim \cD_\cX} \Bigg[ \ell(h(x),0) - \gamma\sum_{g \in \cG} (\lambda_g^+ + \lambda_g^-) \\
    &\quad +f(x)\left( -\ell(h(x),0) + \ell(h(x),1)\left(1 + \sum_{g \in \cG} \lambda_g(g(x)-\beta_g^\fn) \right)\right)  \Bigg]
\end{align*}
\end{lemma}

\begin{proof}
Distributing out like terms in the expression for the Lagrangian in Definition  \ref{def:lag-FN} gives us
\begin{align*}
L_{f}(h,\lambda)
&= \Ex_{x \sim \cD_\cX} \Bigg[ \ell(h(x),0) - \gamma \sum_{g \in \cG} (\lambda_g^+ + \lambda_g^-) \\
    &\quad + f(x)\left( 
    \ell(h(x),1) - \ell(h(x),0) + \ell(h(x),1) \sum_{g \in \cG} (\lambda^+_g(g(x) - \beta_g) + \lambda^-_g(\beta_g - g(x)) \right)\Bigg]\\
    &= \Ex_{x \sim \cD_\cX} \Bigg[ \ell(h(x),0) - \gamma\sum_{g \in \cG} (\lambda_g^+ + \lambda_g^-) \\
    &\quad +f(x)\left( -\ell(h(x),0) + \ell(h(x),1)\left(1 + \sum_{g \in \cG} (\lambda_g^+ - \lambda_g^-)(g(x)-\beta_g) \right)\right)  \Bigg].
\end{align*}
Recall that $\lambda_g = \lambda_g^+ - \lambda_g^-,$ so we are done.
\end{proof}

\begin{lemma}
\label{lem:h-FN}
The optimal post-processed classifier $h$ of $\psi(f, \gamma, \cH_A$) for some regressor $f$ takes the following form:
\[
h(x) = \begin{cases}
1, & \text{if } f(x) > \frac{1}{2 + \sum_{g \in \cG} \lambda_g (g(x)-\beta_g)} \text{ and } 2 + \sum_{g \in \cG} \lambda_g (g(x)-\beta_g) > 0 ,\\
0, & \text{if }f(x) < \frac{1}{2 + \sum_{g \in \cG} \lambda_g (g(x)-\beta_g)} \text { and } 2 + \sum_{g \in \cG} \lambda_g (g(x)-\beta_g) > 0, \\
1, & \text{if } f(x) < \frac{1}{2 + \sum_{g \in \cG} \lambda_g (g(x)-\beta_g)} \text{ and } 2 + \sum_{g \in \cG} \lambda_g (g(x)-\beta_g) < 0 ,\\
0, & \text{if } f(x) > \frac{1}{2 + \sum_{g \in \cG} \lambda_g (g(x)-\beta_g)} \text{ and } 2 + \sum_{g \in \cG} \lambda_g (g(x)-\beta_g) < 0.
\end{cases}
\]

In the edge case in which $f(x) = \frac{1}{2 + \sum_{g \in \cG} \lambda_g (g(x)-\beta_g)}$, $h(x)$ could take either value and might be randomized.
\end{lemma}

\begin{proof}
Note that since we are optimizing over the set of all binary classifiers, $h$ optimizes the Lagrangian objective pointwise for every $x$. In particular, we have from Lemma \ref{lem:expandedLagFN} that:
\[
h(x) =  \arg\min_p \left[\ell(p, 0) + f(x) \left( - \ell(p, 0) + \ell(p, 1) \left(1 + \sum_{g \in \cG} \lambda_g (g(x)-\beta_g)\right)\right) \right].
\]

In order to determine the threshold, we need to check when setting $p=1$ leads to a value less than setting $p=0$. In other words, we need to solve for $f(x)$ when 

\begin{align*}
    1-f(x) &< f(x)\left(1 + \sum_{g \in \cG} \lambda_g(g(x) - \beta_g)\right) \\
    \Rightarrow f(x) &> \frac{1}{2 + \sum_{g \in \cG} \lambda_g(g(x) - \beta_g)}.
\end{align*}

Thus,

\[
h(x) = \begin{cases}
1, & \text{if } f(x) > \frac{1}{2 + \sum_{g \in \cG} \lambda_g (g(x)-\beta_g)} \text{ and } 2 + \sum_{g \in \cG} \lambda_g (g(x)-\beta_g) > 0 ,\\
0, & \text{if }f(x) < \frac{1}{2 + \sum_{g \in \cG} \lambda_g (g(x)-\beta_g)} \text { and } 2 + \sum_{g \in \cG} \lambda_g (g(x)-\beta_g) > 0, \\
1, & \text{if } f(x) < \frac{1}{2 + \sum_{g \in \cG} \lambda_g (g(x)-\beta_g)} \text{ and } 2 + \sum_{g \in \cG} \lambda_g (g(x)-\beta_g) < 0 ,\\
0, & \text{if } f(x) > \frac{1}{2 + \sum_{g \in \cG} \lambda_g (g(x)-\beta_g)} \text{ and } 2 + \sum_{g \in \cG} \lambda_g (g(x)-\beta_g) < 0.
\end{cases}
\]
\end{proof}

From Lemma \ref{lem:h-FN}, we can now define a best-response model and use Algorithm \ref{alg:descentFN} to generate an optimally post-processed model that preserves $\gamma-$False Negative fairness. The algorithm's error bounds may be derived using symmetric arguments to sections 3.1 and 3.2, where $\hat f$ is required to be $\alpha$-approximately jointly multicalibrated in expectation with respect to $s_{\lambda}(x,v) := \mathbbm{1}[\langle \lambda, x_\cG - \beta \rangle \ge (1-2v)/v]$ following the same arguments as used in Lemma \ref{lem:fixh}.

\begin{algorithm}\label{alg:descentFN}
\caption{Projected Gradient Descent Algorithm for $\gamma$-False Negative Fairness}
    \KwInput{$D$: dataset, $f: \cX \to [0, 1]$: regression function, $\cG$: groups, $\gamma$: tolerance on fairness violation, $C$: bound on dual $(\|\lambda\|_1 \leq C)$, $\eta$: learning rate}
    
    Initialize dual vector $\lambda^0 = {\bf 0}$ and set $ T = \frac{1}{4} \cdot C^2 \cdot (C^2 + 4|\cG|)^2 $.
    
    \For{$t = 1, \ldots, T$}
    {
    Primal player updates $h_t$
    \[
    h_t(x) = \begin{cases}
        1, & \text{if } f(x) \ge \frac{1}{2 + \sum_{g \in \cG} \lambda_g^{t-1} (g(x)-\beta_g)} \text{ and } 2 + \sum_{g \in \cG} \lambda_g^{t-1} (g(x)-\beta_g) > 0 ,\\
        0, & \text{if }f(x) < \frac{1}{2 + \sum_{g \in \cG} \lambda_g^{t-1} (g(x)-\beta_g)} \text { and } 2 + \sum_{g \in \cG} \lambda_g^{t-1} (g(x)-\beta_g) > 0, \\
        1, & \text{if } f(x) < \frac{1}{2 + \sum_{g \in \cG} \lambda_g^{t-1} (g(x)-\beta_g)} \text{ and } 2 + \sum_{g \in \cG} \lambda_g^{t-1} (g(x)-\beta_g) < 0 ,\\
        0, & \text{if } f(x) \ge \frac{1}{2 + \sum_{g \in \cG} \lambda_g^{t-1} (g(x)-\beta_g)} \text{ and } 2 + \sum_{g \in \cG} \lambda_g^{t-1} (g(x)-\beta_g) < 0,\\
        0 & \text{if } 2 + \sum_{g \in \cG} \lambda_g^{t-1} (g(x)-\beta_g) = 0
\end{cases}
\]
    Compute
    \begin{align*}
        \hat \rho^t_g& = \E_{(x, y) \sim D} [\ell(h_t(x),1)g(x)f(x)] \text{ for all } g \in \cG,\\
        \hat \rho^t &= \E_{(x, y) \sim D} [\beta_g \ell(h_t(x),1)f(x)], \text{ where } \beta_g = \Pr[g(x) = 1 | y = 0]
    \end{align*}
    Dual player updates 
    \begin{align*}
    \lambda_{g}^{t, +} &= \max(0, \lambda_g^{t,+} + \eta \cdot (\hat \rho^t_g - \hat \rho^t - \gamma)), \\
    \lambda_{g}^{t, -} &= \max(0, \lambda_g^{t,-} + \eta \cdot (\hat \rho^t - \hat \rho^t_g - \gamma)).
    \end{align*}
    
    Dual player sets $\lambda^t =\sum_{g \in \cG} \lambda_g^{t, +} - \lambda_g^{t, -}$.
    
    If $\|\lambda^t\|_1 > C$, set $\lambda^t = \arg \min _{ \{ \tilde{\lambda} \in \mathbb{R}^{2\cG} \vert \Vert\tilde{\lambda}\Vert_1 \le C \}} \Vert \lambda_t - \tilde{\lambda} \Vert_2^2$.
    }
    \KwOutput $\bar{h} := \frac{1}{T} \sum_{t=1}^T \hat{h}_t$, a uniformly random classifier over all rounds' hypotheses.
\end{algorithm}
\subsection{Error Fairness}

\begin{definition}
\noindent We say that classifier $h: \cX \rightarrow \cY$ satisfies $\gamma$-Error (E) Fairness with respect to $\cD$ and $\cG$ if for all $g \in \cG$,

$$w_g^\e \left\vert \err(h,g,\cD) - \err(h,\cD) \right\vert \le \gamma,$$
where $w_g^\e = \Pr_{(x,y)\sim \cD}[g(x)=1]$.
\end{definition}

We consider the following fairness-constrained optimization problem:

\begin{align}
    &\min_{h \in \Delta\cH} && \err(h)  \label{eq:fairnessLP-E} \\ 
    &\textrm{s.t. for each } g \in \cG: &&
    w_g^{\e} \left\vert \err(h,g,\cD) - \err(h,\cD) \right\vert \leq \gamma, \nonumber 
\end{align}

\begin{definition}
\label{def:lp-E}
Let $f: \cX \rightarrow R \subseteq [0,1]$ be some regression function and let $\gamma \in \mathbb{R}_+$. Define $\psi_\e(f,\gamma,\cH)$ to be the following optimization problem:
\begin{align*} 
    &\min_{h \in \Delta \cH} \E_{x \sim \cD_\cX} [f(x)\ell(h(x), 1) + (1-f(x))\ell(h(x),0)] \\
    &\textrm{s.t. for each } g \in \cG: \\
    &\vert \Ex [ \ell(h(x),1)g(x)f^*(x) + \ell(h(x),0)g(x)(1-f^*(x)) \\
    & \quad - w_g^\e (\ell(h(x),1)f^*(x) - w_g^\e \ell(h(x),0)(1-f^*(x) ]\vert \le \gamma,
\end{align*}
where $w_g^{\e} = \Pr_{(x,y)\sim \cD}[g(x)=1]$ as in the previous definition.
\end{definition}

\begin{lemma}
\label{lem:optimal-fair-lp-E}
Let $f^*$ be the Bayes optimal regression function over $\cD$. Then optimization problem $\psi_\e(f^*, \gamma,\cH)$ is equivalent to the fairness-constrained optimization problem \ref{eq:fairnessLP-E}. 
\end{lemma}

\begin{proof}
Note that the objective function is equivalent to that of Equation \ref{eq:fairnessLP}, and hence proof of the objectives being equivalent is identical to that of Lemma \ref{lem:optimal-fair-lp}.
For the constraints, note that 
\begin{align*}
    w_g^\e \vert \err(h,g,\cD)\vert - \err(h,\cD)
    &= \Pr[g(x)=1]\bigg\vert \Pr[y=1 \vert g(x) = 1] \Pr[h(x) = 0 \vert g(x)=1, y=1] \\
    & \quad + \Pr[y=0 \vert g(x) = 1]\Pr[h(x) = 1 \vert g(x)=1, y=0] \\
    & \quad - (\Pr[y=1]\Pr[h(x)=0 \vert y=1] + \Pr[y=0]\Pr[h(x)=1 \vert y=0]) \bigg\vert \\
    &= \Pr[g(x)=1] \bigg\vert \Pr[y=1 \vert g(x)=1] \frac{\Pr[h(x)=0, g(x)=1, y=1]}{\Pr[g(x)=1, y=1]} \\
    & \quad + \Pr[y=1 \vert g(x) = 1] \frac{\Pr[h(x)=1, g(x)=1, y=0]}{\Pr[g(x)=1, y=0]} \\ 
    & \quad - \Pr[y=1] \frac{\Pr[h(x)=0, y=1]}{\Pr[y=1]} - \Pr[y=0] \frac{\Pr[h(x)=1, y=1]}{\Pr[y=0]}\bigg\vert\\
    &= \vert \Ex [ \ell(h(x),1)g(x)f^*(x) + \ell(h(x),0)g(x)(1-f^*(x)) \\
    & \quad - w_g^\e (\ell(h(x),1)f^*(x) - w_g^\e \ell(h(x),0)(1-f^*(x) ]\vert 
\end{align*}
\end{proof}

\begin{definition}[Lagrangian]
\label{def:lag-E}
Given any regression function $f$, we define a Lagrangian of the optimization problem $\psi_\e(f, \gamma, \cH)$ as $L_f^\e:\cH \times \mathbb{R}^{2\vert\cG\vert} \rightarrow \mathbb{R}$:
\begin{align*}
L_f^\e(h, \lambda) &= \Ex_{x \sim \cD_\cX} \bigg[ f(x) \ell(h(x),1) + (1-f(x))\ell(h(x),0) \\
    &\quad + \sum_{g \in \cG} \lambda_g^+ \big( \ell(h(x),1)g(x)f(x) + \ell(h(x),0)g(x)(1-f(x)) \\
    & \quad \quad \quad \quad - w_g^\e \ell(h(x),1)f(x) - w_g^\e \ell(h(x),0)(1-f(x)) - \gamma \big)\\
    &\quad + \sum_{g \in \cG} \lambda_g^- \big( w_g^\e \ell(h(x),1)f(x) + w_g^\e \ell(h(x),0)(1-f(x)) \\
    & \quad \quad \quad \quad - \ell(h(x),1)g(x)f(x) - \ell(h(x),0)g(x)(1-f(x)) - \gamma \big) \bigg].
\end{align*}
\end{definition} 

\begin{lemma}
\label{lem:expandedLag-E}
\begin{align*}
    L^\e_f(h, \lambda) &= \Ex_{x\sim \cD_x} \Bigg[ \ell(h(x),0) \bigg( 1 + \sum_{g \in \cG} \lambda_g (g(x) - w_g^{\e})\bigg) - \gamma \sum_{g \in \cG} (\lambda^+_g + \lambda^-_g)\\
    &\quad + f(x) \Bigg(- \ell(h(x),0) \bigg[ 1 + \sum_{g\in \cG} \lambda_g (g(x) -w_g^{\e}) \bigg] \\
    &\quad \quad \quad \quad \quad \quad \quad + \ell(h(x),1) \bigg[ 1 + \sum_{g \in \cG} \lambda_g (g(x) - w_g^{\e}) \bigg] \Bigg) \Bigg]
\end{align*}
\end{lemma}

\begin{proof}
Distribute out like terms as shown previously.
\end{proof}

\begin{lemma}
\label{lem:h-E}
The optimal post-processed classifier $h$ of $\psi(f, \gamma, \cH_A$) for some regressor $f$ takes the following form:

\[
h(x) = \begin{cases}
1, & \text{if } f(x) > \frac{1 + \sum_{g \in \cG} \lambda_g (g(x) - w_g^{\e})}{ 2 + 2\sum_{g \in \cG} \lambda_g (g(x)-w_g^{\e})} \text{ and } 2 + 2\sum_{g \in \cG} \lambda_g (g(x)-w_g^{\e})> 0 ,\\
0, & \text{if }f(x) < \frac{1 + \sum_{g \in \cG} \lambda_g (g(x) - w_g^{\e})}{ 2 + 2\sum_{g \in \cG} \lambda_g (g(x)-w_g^{\e})} \text{ and } 2 + 2\sum_{g \in \cG} \lambda_g (g(x)-w_g^{\e})  > 0, \\
1, & \text{if } f(x) < \frac{1 + \sum_{g \in \cG} \lambda_g (g(x) - w_g^{\e})}{ 2 + 2\sum_{g \in \cG} \lambda_g (g(x)-w_g^{\e})} \text{ and } 2 + 2\sum_{g \in \cG} \lambda_g (g(x)-w_g^{\e}) < 0 ,\\
0, & \text{if } f(x) > \frac{1 + \sum_{g \in \cG} \lambda_g (g(x) - w_g^{\e})}{ 2 + 2\sum_{g \in \cG} \lambda_g (g(x)-w_g^{\e})} \text{ and } 2 + 2\sum_{g \in \cG} \lambda_g (g(x)-w_g^{\e}) >0 .
\end{cases}
\]

In the edge case in which $f(x) = \frac{1 + \sum_{g \in \cG} \lambda_g (g(x) - w_g^{\e})}{ 2 + 2\sum_{g \in \cG} \lambda_g (g(x)-w_g^{\e})}, h(x) $ could take either value and might be randomized.
\end{lemma}

\begin{proof}
Note that since we are optimizing over the set of all binary classifiers, $h$ optimizes the Lagrangian objective pointwise for every $x$. In particular, we have from Lemma \ref{lem:expandedLag-E} that:

\begin{align*}
h(x) &= \arg\min_p \Bigg[ \ell(p,0) \bigg( 1 + \sum_{g \in \cG} \lambda_g (g(x) - w_g^{\e})\bigg) \\
    &\quad + f(x) \Bigg(- \ell(p,0) \bigg[ 1 + \sum_{g\in \cG} \lambda_g (g(x) -w_g^{\e}) \bigg] + \ell(p,1) \bigg[ 1 + \sum_{g \in \cG} \lambda_g (g(x) - w_g^{\e}) \bigg] \Bigg) \Bigg]
\end{align*}

Setting $p=0$ makes the inner portion of the expression evaluate to

\[
    f(x) \Bigg(  1 + \sum_{g \in \cG} \lambda_g (g(x)-w_g^{\e})\Bigg),\]

and setting $p=1$ makes the inner portion of the expression evaluate to
\[
\bigg( 1 + \sum_{g \in \cG} \lambda_g (g(x) - w_g^{\e})\bigg) - f(x) \Bigg(1 + \sum_{g\in \cG} \lambda_g (g(x) -w_g^{\e})\Bigg)
\]

In order to find the optimal $h$, we want to find the threshold at which setting $p=1$ minimizes the expression, and hence:

\begin{align*}
    \bigg( 1 + \sum_{g \in \cG} \lambda_g (g(x) - w_g^{\e})\bigg) - f(x) \Bigg(1 + \sum_{g\in \cG} \lambda_g (g(x) -w_g^{\e})\Bigg) &<  f(x) \Bigg(  1 + \sum_{g \in \cG} \lambda_g (g(x)-w_g^{\e})\Bigg)\\
    \frac{1 + \sum_{g \in \cG} \lambda_g (g(x) - w_g^{\e})}{\Bigg(  1 + \sum_{g \in \cG} \lambda_g (g(x)-w_g^{\e})\Bigg) + \Bigg(1 + \sum_{g\in \cG} \lambda_g (g(x) -w_g^{\e})\Bigg)} &< f(x) \\
    \frac{1 + \sum_{g \in \cG} \lambda_g (g(x) - w_g^{\e})}{ 2 + 2\sum_{g \in \cG} \lambda_g (g(x)-w_g^{\e})}  &< f(x) \\
\end{align*}

Thus, 
\[
h(x) = \begin{cases}
1, & \text{if } f(x) > \frac{1 + \sum_{g \in \cG} \lambda_g (g(x) - w_g^{\e})}{ 2 + 2\sum_{g \in \cG} \lambda_g (g(x)-w_g^{\e})} \text{ and } 2 + 2\sum_{g \in \cG} \lambda_g (g(x)-w_g^{\e})> 0 ,\\
0, & \text{if }f(x) < \frac{1 + \sum_{g \in \cG} \lambda_g (g(x) - w_g^{\e})}{ 2 + 2\sum_{g \in \cG} \lambda_g (g(x)-w_g^{\e})} \text{ and } 2 + 2\sum_{g \in \cG} \lambda_g (g(x)-w_g^{\e})  > 0, \\
1, & \text{if } f(x) < \frac{1 + \sum_{g \in \cG} \lambda_g (g(x) - w_g^{\e})}{ 2 + 2\sum_{g \in \cG} \lambda_g (g(x)-w_g^{\e})} \text{ and } 2 + 2\sum_{g \in \cG} \lambda_g (g(x)-w_g^{\e}) < 0 ,\\
0, & \text{if } f(x) > \frac{1 + \sum_{g \in \cG} \lambda_g (g(x) - w_g^{\e})}{ 2 + 2\sum_{g \in \cG} \lambda_g (g(x)-w_g^{\e})} \text{ and } 2 + 2\sum_{g \in \cG} \lambda_g (g(x)-w_g^{\e}) >0 .
\end{cases}
\]
\end{proof}

From Lemma \ref{lem:h-E}, we can now define a best-response model and use Algorithm \ref{alg:descentE} to generate an optimally post-processed model that preserves $\gamma-$Error fairness. The algorithm's error bounds may be derived using symmetric arguments to sections 3.1 and 3.2, where $\hat{f}$ is $\alpha$-multicalibrated in expectation with respect to $\cG, \cH,$ and $\cG \times \cH$ and is jointly multicalibrated with respect to functions of the form:

\[
\mathbbm{1}\left[\langle \lambda^{t-1}, x_\cG - w^{\e} \rangle \ge \frac{2v-1}{1-2v}\right] \]
the proofs from section 3.2 may be modified to get its desired error bounds.

\begin{algorithm}\label{alg:descentE}
\caption{Projected Gradient Descent Algorithm for $\gamma$-Error Fairness}
    \KwInput{$D$: dataset, $f: \cX \to [0, 1]$: regression function, $\cG$: groups, $\gamma$: tolerance on fairness violation, $C$: bound on dual $(\|\lambda\|_1 \leq C)$, $\eta$: learning rate}
    
    Initialize dual vector $\lambda^0 = {\bf 0}$ and set $ T = \frac{1}{4} \cdot C^2 \cdot (C^2 + 4|\cG|)^2 $.
    
    \For{$t = 1, \ldots, T$}
    {
    Primal player updates $h_t$

\[
h_t(x) = \begin{cases}
1, & \text{if } f(x) > \frac{1 + \sum_{g \in \cG} \lambda_g^{t-1} (g(x) - w_g^{\e})}{ 2 + 2\sum_{g \in \cG} \lambda_g^{t-1} (g(x)-w_g^{\e})} \text{ and } 2 + 2\sum_{g \in \cG} \lambda_g^{t-1} (g(x)-w_g^{\e}) > 0 ,\\
0, & \text{if }f(x) < \frac{1 + \sum_{g \in \cG} \lambda_g^{t-1} (g(x) - w_g^{\e})}{ 2 + 2\sum_{g \in \cG} \lambda_g^{t-1} (g(x)-w_g^{\e})} \text{ and } 2 + 2\sum_{g^{t-1} \in \cG} \lambda_g (g(x)-w_g^{\e})  > 0, \\
1, & \text{if } f(x) < \frac{1 + \sum_{g \in \cG} \lambda_g^{t-1} (g(x) - w_g^{\e})}{ 2 + 2\sum_{g \in \cG} \lambda_g^{t-1} (g(x)-w_g^{\e})} \text{ and } 2 + 2\sum_{g \in \cG} \lambda_g^{t-1} (g(x)-w_g^{\e}) < 0 ,\\
0, & \text{if } f(x) > \frac{1 + \sum_{g \in \cG} \lambda_g^{t-1} (g(x) - w_g^{\e})}{ 2 + 2\sum_{g \in \cG} \lambda_g^{t-1} (g(x)-w_g^{\e})} \text{ and } 2 + 2\sum_{g \in \cG} \lambda_g^{t-1} (g(x)-w_g^{\e}) >0, \\
1, & \text{if } 2 + 2\sum_{g \in \cG} \lambda_g^{t-1} (g(x)-w_g^{\e}) = 0.
\end{cases}
\]
    Compute
    \begin{align*}
        \hat \rho^t_g& = \Ex_{(x,y)\sim D} [ \ell(h_t(x),1)g(x)f(x) + \ell(h_t(x),0)g(x)(1-f(x)) \\
    & \quad - w_g^\e (\ell(h_t(x),1)f(x) - w_g^\e \ell(h_t(x),0)(1-f(x) ] \text{ for all } g \in \cG,\\
    \hat \rho^t &= \Ex_{(x, y) \sim D} [f(x)\ell(h_t(x), 1) + (1-f(x))\ell(h_t(x),0)],
    \end{align*}

    where $w_g^{\e} = \Pr_{(x,y)\sim \cD}[g(x)=1]$.

    Dual player updates 
    \begin{align*}
    \lambda_{g}^{t, +} &= \max(0, \lambda_g^{t,+} + \eta \cdot (\hat \rho^t_g - \hat \rho^t - \gamma)), \\
    \lambda_{g}^{t, -} &= \max(0, \lambda_g^{t,-} + \eta \cdot (\hat \rho^t - \hat \rho^t_g - \gamma)).
    \end{align*}
    
    Dual player sets $\lambda^t =\sum_{g \in \cG} \lambda_g^{t, +} - \lambda_g^{t, -}$.
    
    If $\|\lambda^t\|_1 > C$, set $\lambda^t = \arg \min _{ \{ \tilde{\lambda} \in \mathbb{R}^{2\cG} \vert \Vert\tilde{\lambda}\Vert_1 \le C \}} \Vert \lambda_t - \tilde{\lambda} \Vert_2^2$.
    }
    \KwOutput $\bar{h} := \frac{1}{T} \sum_{t=1}^T \hat{h}_t$, a uniformly random classifier over all rounds' hypotheses.
\end{algorithm}
\subsection{Statistical Parity Fairness}

\begin{definition}
\noindent We say that classifier $h: \cX \rightarrow \cY$ satisfies $\gamma$-Statistical Parity (SP) Fairness with respect to $\cD$ and $\cG$ if for all $g \in \cG$,
\[\Pr_{(x,y) \sim \cD} [g(x) = 1] \left\vert\E_{(x,y) \sim \cD}[h(x) \vert g(x)=1] - \E_{(x,y) \sim \cD}[h(x)] \right\vert \in \gamma.\]
\end{definition}

We consider the following fairness-constrained optimization problem:

\begin{align}
    &\min_{h \in \Delta\cH} && \err(h)  \label{eq:fairnessLP-SP} \\ 
    &\textrm{s.t. for each } g \in \cG: &&
    \Pr[g(x)=1] \left\vert \E_{(x,y) \sim \cD}[h(x) \vert g(x)=1] - \E_{(x,y) \sim \cD}[h(x)] \right\vert \leq \gamma, \nonumber 
\end{align}

\begin{definition}
Let $f:\cX \rightarrow R \subseteq [0,1]$ be some regression function and let $\gamma \in \mathbb{R}_+$. Define $\phi_\SP(f, \gamma, \cH)$ to be the following optimization problem:
\begin{align*} 
    &\min_{h \in \Delta \cH} &&\E_{x \sim \cD_\cX} [f(x)\ell(h(x), 1) + (1-f(x))\ell(h(x),0)] \\
    &\textrm{s.t. for each } g \in \cG: && \left\vert \E_{x \sim \cD_\cX} [h(x)g(x)] - w_g^\SP \E_{x \sim \cD_\cX}[h(x)] \right\vert \le \gamma
\end{align*}
where $w_g^\SP = \Pr[g(x)=1]$.
\end{definition}

\begin{lemma}
\label{lem:optimal-fair-lp-SP}
Let $f^*$ be the Bayes optimal regression function over $\cD$. Then optimization problem $\psi_\SP(f^*, \gamma,\cH)$ is equivalent to the fairness-constrained optimization problem \ref{eq:fairnessLP-SP}. 
\end{lemma}

\begin{definition}[Lagrangian]
\label{def:lag-SP}
Given any regression function $f$, we define a Lagrangian of the optimization problem $\psi_\SP(f, \gamma, \cH)$ as $L_f^\SP:\cH \times \mathbb{R}^{2\vert\cG\vert} \rightarrow \mathbb{R}$:
\begin{align*}
L_f^\SP(h, \lambda) &= \E_{x \sim \cD_\cX} \bigg[ f(x)\ell(h(x),1) + (1-f(x))(\ell(h(x),0) \\
& \quad \quad + \sum_{g \in \cG} \lambda^+_g \big(h(x)(g(x)-1) - \gamma\big) + \sum_{g \in \cG} \lambda^-_g \big(h(x)(1-g(x)\big) - \gamma) \bigg].
\end{align*}
\end{definition} 

\begin{lemma}
The optimal post-processed classifier $h$ of $\psi_\SP(f, \gamma, \cH_A$) for some regressor $f$ takes the following form:

\[
h(x) = \begin{cases}
1, & \text{if } f(x) > 1/2 + (1/2)\sum_{g\in \cG} \lambda_g(g(x)-1),\\
0, & \text{if }f(x) < 1/2 + (1/2)\sum_{g\in \cG} \lambda_g(g(x)-1).
\end{cases}
\]

In the edge case in which $f(x) = 1/2 + (1/2)\sum_{g\in \cG} \lambda_g(g(x)-1)$, $h(x)$ could take either value and might be randomized.
\end{lemma}

\begin{proof}
Note that we can rewrite our Lagrangian from Definition \ref{def:lag-SP} as

\[ 
L_f^\SP(h, \lambda) = \E_{x \sim \cD_\cX} \left[ f(x)(\ell(h(x),1) - \ell(h(x),0)) + \ell(h(x),0) + h(x) \sum_{g\in \cG} \lambda_g(g(x)-1) + \gamma \sum_{g \in \cG} (\lambda^+ + \lambda^-)\right],
\]

and hence our optimal $h$ will be optimal pointwise, i.e. 

\[
h(x) \arg\min_p \left[ f(x)(\ell(p,1) + \ell(p,0)) - \ell(p,0) + p \sum_{g\in \cG} \lambda_g (g(x)-1)\right]
\]

We can then find our threshold by comparing this expression when $p=0$ and $p=1$, i.e. 

\begin{align*}
    -f(x) + 1 + \sum_{g\in \cG} \lambda_g(g(x)-1) &< f(x) \\
    \frac{1 + \sum_{g\in \cG} \lambda_g(g(x)-1)}{2} &< f(x). \\
\end{align*}

Hence,

\[
h(x) = \begin{cases}
1, & \text{if } f(x) > 1/2 + (1/2)\sum_{g\in \cG} \lambda_g(g(x)-1),\\
0, & \text{if }f(x) < 1/2 + (1/2)\sum_{g\in \cG} \lambda_g(g(x)-1).
\end{cases}
\]

\end{proof}

We can now define a best-response model and use Algorithm \ref{alg:descentSP} to generate an optimally post-processed model that preserves $\gamma$-Statistical Parity fairness. Assuming that $\hat{f}$ is $\alpha$-multicalibrated in expectation with respect to $\cG, \cH,$ and $\cG \times \cH$ and is jointly multicalibrated with respect to functions of the form $
\mathbbm{1}[\langle \lambda, x_\cG - \beta \rangle \ge 2v-1]$, the proofs from section 3.2 may be modified to get its desired error bounds.

\begin{algorithm}\label{alg:descentSP}
\caption{Projected Gradient Descent Algorithm for $\gamma$-Statistical Parity Fairness}
    \KwInput{$D$: dataset, $f: \cX \to [0, 1]$: regression function, $\cG$: groups, $\gamma$: tolerance on fairness violation, $C$: bound on dual $(\|\lambda\|_1 \leq C)$, $\eta$: learning rate}
    
    Initialize dual vector $\lambda^0 = {\bf 0}$ and set $ T = \frac{1}{4} \cdot C^2 \cdot (C^2 + 4|\cG|)^2 $.
    
    \For{$t = 1, \ldots, T$}
    {
    Primal player updates $h_t$

    \[
        h_t(x) = \begin{cases}
        1, & \text{if } f(x) \ge 1/2 + (1/2)\sum_{g\in \cG} \lambda_g^{t-1}(g(x)-1),\\
        0, & \text{if }f(x) < 1/2 + (1/2)\sum_{g\in \cG} \lambda_g^{t-1}(g(x)-1).
        \end{cases}
    \]

    Compute

    \begin{align*}
        \hat \rho^t_g& =  \left\vert \E_{x \sim \cD_\cX} [h_t(x)g(x)] - w_g^\SP \E_{x \sim \cD_\cX}[h_t(x)] \right\vert \text{ for all } g \in \cG,\\
    \hat \rho^t &= \Ex_{(x, y) \sim D} [f(x)\ell(h_t(x), 1) + (1-f(x))\ell(h_t(x),0)],
    \end{align*}

    where $w_g^\SP = \Pr[g(x)=1]$.\\
    Dual player updates 
    \begin{align*}
    \lambda_{g}^{t, +} &= \max(0, \lambda_g^{t,+} + \eta \cdot (\hat \rho^t_g - \hat \rho^t - \gamma)), \\
    \lambda_{g}^{t, -} &= \max(0, \lambda_g^{t,-} + \eta \cdot (\hat \rho^t - \hat \rho^t_g - \gamma)).
    \end{align*}

    Dual player sets $\lambda^t =\sum_{g \in \cG} \lambda_g^{t, +} - \lambda_g^{t, -}$.
    
    If $\|\lambda^t\|_1 > C$, set $\lambda^t = \arg \min _{ \{ \tilde{\lambda} \in \mathbb{R}^{2\cG} \vert \Vert\tilde{\lambda}\Vert_1 \le C \}} \Vert \lambda_t - \tilde{\lambda} \Vert_2^2$.
    }
    \KwOutput $\bar{h} := \frac{1}{T} \sum_{t=1}^T \hat{h}_t$, a uniformly random classifier over all rounds' hypotheses.
\end{algorithm}




\subsection{Achieving All Fairness Notions}

Ideally, we would like our function to be multicalibrated so that we can achieve any fairness notion downstream. Putting everything together from the previous sections, we can do so.

\begin{definition}[Set of thresholding functions $\cB$] Let $x_\cG \in \{0,1\}^{\vert \cG \vert}$ denote the group membership indicator vector of some point $x$, and define the following functions:
\begin{align*}
    d^{\fp}(v) &:= \frac{2v-1}{1-v},\\
    d^{\fn}(v) &:= \frac{1-2v}{v},\\
    d^{\e}(v) &:= \frac{2v-1}{1-2v},\\
    d^{\SP}(v) &:= 2v-1.
\end{align*}
Then, for any $\lambda, x, \beta$, let 
\begin{align*}
    s^{\fp}_{\lambda}(x,v) &:= \mathbbm{1}[\langle \lambda, x_\cG - \beta^{\fp} \rangle \ge d^{\fp, \e}(v)],\\
    s^{\fn}_{\lambda}(x,v) &:= \mathbbm{1}[\langle \lambda, x_\cG - \beta^{\fn} \rangle \ge d^{\fn}(v)],\\
    s^{\e}_{\lambda}(x,v) &:= \mathbbm{1}[\langle \lambda, \alpha x_\cG - w^{\e} \rangle \ge d^{\e}(v)],\\
    s^{\SP}_{\lambda}(x,v) &:= \mathbbm{1}[\langle \lambda, x_\cG - 1 \rangle \ge d^{\SP}(v)],
\end{align*}

where 
\begin{align*}
\beta^{\fp} &= \{\Pr_{(x,y)\sim \cD}\left[g(x)=1 \vert y=0 \right]\}_{g\in\cG}, \\
\beta^{\fn} &= \{\Pr_{(x,y)\sim \cD}\left[g(x)=1 \vert y=1 \right]\}_{g\in\cG},\\
w^{\e} &= \{\Pr_{(x,y)\sim \cD}[g(x)=1]\}_{g\in\cG}.
\end{align*}
Define $\cB = \{s^{\fp}_{\lambda} | \lambda \in \Lambda(C)\} \cup \{s^{\fn}_{\lambda} | \lambda \in \Lambda(C)\} \cup \{s^{\e}_{\lambda} | \lambda \in \Lambda(C)\} \cup \{s^{\SP}_{\lambda} | \lambda \in \Lambda(C)\}$, where $\Lambda(C) = \left\{ \lambda \in \mathbb{R}^{2\mathcal{G}}  \big\vert \Vert \lambda \Vert_1 \le C \right\}$, as defined in Equation \ref{eq:game}.
\end{definition}

Then, if $f$ is multicalibrated with respect to $\cB$, any of the projected gradient descent algorithms covered above (Algorithms \ref{alg:descent} through \ref{alg:descentE}) may be run to achieve the desired fairness notion.  

\ifarxiv
\else
\section{Expanded Proofs and Section 3 Discussion}
\label{ap:proofs}
\optimalfair*
\begin{proof}
We confirm that the objective and constraints are both equivalent. First the objective: 
\begin{align*}
    err(h) &= \E_{(x,y)\sim \cD} \left[\ell(h(x),y) \right] \\
        &= \sum_{(x,y) \in \cX\times \cY} \Pr\left(X=x, Y=y\right) \ell(h(x),y) \\
        &= \sum_{x \in \cX} \Pr\left(X=x, Y=0 \right)\ell(h(x),0) + \Pr\left(X=x, Y=1 \right)\ell(h(x),1) \\
        &= \E_{x \in \cX} [(1-f^*(x))\ell(h(x),0) + f^*(x)\ell(h(x),1)]
\end{align*}
For the constraints, note that 
\begin{align*}
    w_g \vert \rho_g(h) - \rho(h) \vert &= \Pr[g(x)=1, y=0] \left\vert \Pr[h(x) = 1 \vert g(x)=1, y=0] - \Pr[h(x)=1 \vert y=0] \right\vert \\
    &= \Pr[g(x)=1, y=0] \bigg\vert \frac{\Pr[h(x)=1, g(x)=1, y=0]}{\Pr[g(x)=1, y=0]} - \frac{\Pr[h(x)=1, y=0]}{\Pr[Y=0]}\bigg\vert\\
    &= \left\vert \Pr[h(x)=1, g(x)=1, y=0] - \frac{\Pr[g(x)=1, y=0] \Pr[h(x)=1, y=0]}{\Pr[Y=0]}  \right\vert \\
    &= \left\vert \Ex[\ell(h(x),0)g(x)(1-f^*(x))] - \frac{\Pr[g(x)=1, y=0]}{\Pr[Y=0]} \Ex\left[\ell(h(x),0)(1-f^*(x))\right] \right\vert \\
     &= \left\vert \Ex[\ell(h(x),0)g(x)(1-f^*(x))] - \Pr[g(x)=1 \vert Y=0] \Ex\left[\ell(h(x),0)(1-f^*(x))\right] \right\vert \\
     &= \left\vert \Ex[\ell(h(x),0)g(x)(1-f^*(x))] - \beta_g \Ex\left[\ell(h(x),0)(1-f^*(x))\right] \right\vert.
\end{align*}
The result follows. 
\end{proof}

\begin{lemma}
\label{lem:expandedLag}
\begin{align*}
L_{f}(h,\lambda) &=  \E_{x \sim \cD_\cX} \Bigg[\ell(h(x), 0) \Big( 1 + \sum_{g \in \cG} \lambda_g (g(x)-\beta_g) \Big) - \gamma \sum_{g \in \cG} (\lambda_g^+ + \lambda_g^-) \\
& \quad  - f(x) \Big( - \ell(h(x), 1) + \ell(h(x), 0) \big(1 + \sum_{g \in \cG} \lambda_g (g(x)-\beta_g)\big)\Big)\Bigg].
\end{align*}
\end{lemma}
\begin{proof}
Distributing out like terms in the expression for the Lagrangian in Definition  \ref{def:lag} gives us
\begin{align*}
L_{f}(h,\lambda)
&=
\E_{x \sim \cD_\cX} \Bigg[
f(x)\ell(h(x), 1) + (1-f(x))\ell(h(x),0) \\ 
& \quad + \sum_{g \in \cG}\lambda_g^+ \big(\ell(h(x),0) g(x)(1-f(x)) - \beta_g \ell(h(x), 0) (1-f(x)) - \gamma \big) \\
& \quad + \lambda_g^-\big(\beta_g \ell(h(x), 0) (1 - f(x)) - \ell(h(x),0)g(x)(1-f(x)) - \gamma \big) 
\Bigg]  \\
& = \E_{x \sim \cD_\cX} \Bigg[ \ell(h(x), 0) \Big( 1 + \sum_{g \in \cG} \lambda_g^+ (g(x) - \beta_g) + \lambda_g^- ( \beta_g- g(x) ) \Big) - \gamma \sum_{g \in \cG} (\lambda_g^+ + \lambda_g^-) \\
& \quad  - f(x) \Big( - \ell(h(x), 1) + \ell(h(x), 0) \big( 1 + \sum_{g \in \cG} \lambda^+_g (g(x)-\beta_g) + \sum_{g \in \cG} \lambda^-_g (\beta_g - g(x)) \big) \Big) \Bigg]\\
& = \E_{x \sim \cD_\cX} \Bigg[\ell(h(x), 0) \Big( 1 + \sum_{g \in \cG} (\lambda^+_g-\lambda^-_g) (g(x)-\beta_g) \Big) - \gamma \sum_{g \in \cG} (\lambda_g^+ + \lambda_g^-) \\
& \quad  - f(x) \Big( - \ell(h(x), 1) + \ell(h(x), 0) \big(1 + \sum_{g \in \cG} (\lambda^+_g-\lambda^-_g) (g(x)-\beta_g)\big)\Big)\Bigg].
\end{align*}
Recall that $\lambda_g = \lambda_g^+ - \lambda_g^-,$ so we are done.
\end{proof}

\lemh*

\begin{proof}
Note that since we are optimizing over the set of all binary classifiers, $h$ optimizes the Lagrangian objective pointwise for every $x$. In particular, we have from Lemma \ref{lem:expandedLag} that:
\[
h(x) =  \arg\min_p \Bigg[\ell(p, 0) \Big( 1 + \sum_{g \in \cG} \lambda_g (g(x)-\beta_g) \Big) - f(x) \Big( - \ell(p, 1) + \ell(p, 0) \big(1 + \sum_{g \in \cG} \lambda_g (g(x)-\beta_g)\big)\Big)\Bigg].
\]

Determining the optimal threshold is equivalent to determining when the above expression with $\ell(p,0)=1$ and $\ell(p,1)=0$ is less than $f(x)$, i.e.

\begin{align*}
    1 + \sum_{g \in \cG} \lambda_g (g(x)-\beta_g) - f(x) \big(1 + \sum_{g \in \cG} \lambda_g (g(x)-\beta_g)\big) &< f(x)\\
     1 + \sum_{g \in \cG} \lambda_g (g(x)-\beta_g) &< f(x)\left(1 + \big(1 + \sum_{g \in \cG} \lambda_g (g(x)-\beta_g)\big) \right).
\end{align*}

Thus,

\[
h(x) = \begin{cases}
1, & \text{if } f(x) > \frac{1 + \sum_{g \in \cG} \lambda_g (g(x)-\beta_g)}{2 + \sum_{g \in \cG} \lambda_g (g(x)-\beta_g)} \text{ and } (2 + \sum_{g \in \cG} \lambda_g (g(x)-\beta_g)) > 0 ,\\
0, & \text{if }f(x) < \frac{1, + \sum_{g \in \cG} \lambda_g (g(x)-\beta_g)}{2 + \sum_{g \in \cG} \lambda_g (g(x)-\beta_g)} \text { and } (2 + \sum_{g \in \cG} \lambda_g (g(x)-\beta_g)) > 0 \\
1, & \text{if } f(x) < \frac{1 + \sum_{g \in \cG} \lambda_g (g(x)-\beta_g)}{2 + \sum_{g \in \cG} \lambda_g (g(x)-\beta_g)} \text{ and } (2 + \sum_{g \in \cG} \lambda_g (g(x)-\beta_g)) < 0 ,\\
0, & \text{if } f(x) > \frac{1, + \sum_{g \in \cG} \lambda_g (g(x)-\beta_g)}{2 + \sum_{g \in \cG} \lambda_g (g(x)-\beta_g)} \text{ and } (2 + \sum_{g \in \cG} \lambda_g (g(x)-\beta_g)) < 0
\end{cases}
\]
\end{proof}

In Lemma \ref{lem:h}, we can only describe the optimal post-processed classifier for cases where either $f(x)$ is less than or greater than the threshold $\frac{1 + \sum_{g \in \cG} \lambda_g (g(x)-\beta_g)}{2 + \sum_{g \in \cG} \lambda_g (g(x)-\beta_g)}$, $h(x)$. In practice, our algorithm will need to update $h$ at round $t$ according to the current dual variables $\lambda$ in a way that is well-defined for all values of $f(x)$. Hence, we define our best response as follows, where ties between $f(x)$ and the threshold are broken by rounding to 1. 

\begin{definition}[Best Response Model]
Given regressor $f$ and dual variables $\lambda$, let the best response $h$ be defined as
\label{def:BR_h}
\[
        h(x) = \begin{cases}
        1, & \text{if } f(x) \geq \frac{1 + \sum_{g \in \cG} \lambda_g (g(x)-\beta_g)}{2 + \sum_{g \in \cG} \lambda_g (g(x)-\beta_g)} \text{ and } (2 + \sum_{g \in \cG} \lambda_g (g(x)-\beta_g)) > 0,\\
        0, & \text{if } f(x) < \frac{1 + \sum_{g \in \cG} \lambda_g (g(x)-\beta_g)}{2 + \sum_{g \in \cG} \lambda_g (g(x)-\beta_g)} \text{ and } (2 + \sum_{g \in \cG} \lambda_g (g(x)-\beta_g)) > 0,\\
        1, & \text{if } f(x) \leq \frac{1 + \sum_{g \in \cG} \lambda_g (g(x)-\beta_g)}{2 + \sum_{g \in \cG} \lambda_g (g(x)-\beta_g)} \text{ and } (2 + \sum_{g \in \cG} \lambda_g (g(x)-\beta_g)) < 0,\\
        0, & \text{if } f(x) > \frac{1 + \sum_{g \in \cG} \lambda_g (g(x)-\beta_g)}{2 + \sum_{g \in \cG} \lambda_g (g(x)-\beta_g)} \text{ and } (2 + \sum_{g \in \cG} \lambda_g (g(x)-\beta_g)) < 0.
        \end{cases}
    \]
\end{definition}

\begin{lemma}
\label{lem:br_h}
For any regression model $f$ and dual variables $\lambda$, The classifier $h$ defined in Definition \ref{def:BR_h} is a ``best response'' in the sense that:
$$h \in \arg\min_{h \in \cH_A} L_f(h,\lambda).$$
\end{lemma}

\subsection{Proofs from Section 3.1}

\algmain*
\regret*

To prove this, we will use the following result from Freund and Shapire.

\begin{theorem}[\cite{freund96}](Approximately solving a game).
If $\lambda_1, \ldots, \lambda_T \in \Delta_{\lambda}$ is the sequence of distributions over $\lambda$ played by the dual player and $h_1, \ldots, h_T \in \cH$ is the sequence of best-response hypotheses played by the primal player satisfying regret guarantees
\begin{align*}
\frac{1}{T} \max_{\lambda \in \Lambda} \sum_{t = 1}^T U(h_t, \lambda) &-
\frac{1}{T}\sum_{t = 1}^T \E_{\lambda \sim \lambda_t}[U(h_t, \lambda)] \leq \Delta_1 \\
&\textrm{and} \\
\frac{1}{T} \sum_{t = 1}^T \E_{\lambda \sim \lambda_t}[U(h_t, \lambda)] &-
\frac{1}{T} \min_{h \in \cH} \sum_{t = 1}^T \E_{\lambda \sim \lambda_t}[U(h, \lambda)] \leq \Delta_2
\end{align*}
then the time-average of the two players' empirical distributions is a $(\Delta_1 + \Delta_2)-$approximate equilibrium.
\label{thm:freundshapire}
\end{theorem}

\begin{proof}[Proof of Theorem \ref{thm:regret}]
We follow the regret analysis of \cite{zinkevich2003online}. To instantiate their result, we need a bound on the norm of the gradients of the loss function and on the diameter of the feasible set $F$. 
First, we see that at each step the gradient of the loss seen by gradient descent is bounded:
\begin{align*}
    \|\nabla \ell \|^2 = \sum_{g \in \cG} w_g\left( \rho_g -  \rho - \gamma \right)^2+ w_g\left(-\rho_g + \rho - \gamma \right)^2 \leq 2|\cG|.
\end{align*}
Second, we see that if we consider the feasible set such that $\|\lambda\| \leq \frac{1}{\epsilon}$, then $\|F\|^2 = \frac{1}{\epsilon^2}$.
Thus we have that the regret of the dual player is bounded:
\begin{align*}
    \mathcal{R}(T) &\leq \frac{\|F\|^2 \sqrt{T}}{2} + (\sqrt{T} - \frac12 ) \|\nabla \ell\|^2 \\
    \frac{\mathcal{R}(T)}{T} &\leq \frac{1}{T} \left( \frac{\frac{1}{\epsilon^2} \sqrt{T}}{2} + (\sqrt{T} - \frac12) 2 |\cG| \right) \leq \frac{\frac{1}{\epsilon^2} + 4|\cG|}{2\sqrt{T}}.
\end{align*}
After $T = \frac{1}{4\epsilon^2}\left(\frac{1}{\epsilon^2} + 4|\cG| \right)^2$ rounds, by \cite{freund96} the average over empirical distributions of play of the dual and primal players, $\bar{\lambda}$ and $\bar{h}$, respectively, form an $\epsilon-$approximate equilibrium solution to the zero-sum game defined by \ref{eq:game}.
\end{proof}

\begin{proof}[Proof of Theorem \ref{thm:alg-main}]
Applying Theorems \ref{thm:approxminmax} and \ref{thm:regret}, we have that after $T$ rounds $(\bar{h}, \bar{\lambda})$ is an $\epsilon$-approximate equilibrium to the zero-sum game of \ref{eq:game} and equivalently a minimax solution to the $\Lambda$-bounded Lagrangian. Taking $\epsilon= 1/C,$
the solution $(\bar{h}, \bar{\lambda})$ is a $\frac{1 + 2\epsilon}{1/\epsilon} = 1/C + 2/C^2$ approximate solution to the original linear program \ref{eq:fairnessLP}. 
\end{proof}

\subsection{Proofs from Section 3.2}

\fixh*

\begin{proof}
Recall from Lemma \ref{lem:br_h} and Algorithm \ref{alg:descent} that the best response to $\lambda$ that the primal player can make is to compute $h$ based on the thresholding of the expression 
\[
\tau = \frac{1 + \sum_{g \in \cG} \lambda_g^{t-1} (g(x)-\beta_g)}{2 + \sum_{g \in \cG} \lambda_g^{t-1} (g(x)-\beta_g)}.
\]
Setting this threshold to be greater than  or equal to some value $v$, note the following is implied:
\begin{align*}
\frac{1 + \sum_{g \in \cG} \lambda_g^{t-1} (g(x) - \beta_g)}{2 + \sum_{g \in \cG} \lambda_g^{t-1} (g(x) - \beta_g)} &\ge v,\\
\Rightarrow \sum_{g \in \cG} \lambda_g^{t-1} (g(x) - \beta_g) - v \sum_{g \in \cG} \lambda_g^{t-1} (g(x) - \beta_g) &\ge 2v - 1, \\
\Rightarrow (1 - v) (\sum_{g \in \cG}  \lambda_g^{t-1}(g(x) - \beta_g) &\ge 2v - 1, \\
\Rightarrow \langle \lambda^{t-1}, x_\cG - \beta \rangle = \sum_{g \in \cG} \lambda_g^{t-1} (g(x) - \beta_g) &\ge \frac{2v - 1}{1 - v}.
\end{align*}
Thus, taking the indicator of  
\[
\mathbbm{1}[\langle \lambda^{t-1}, x_\cG - \beta \rangle \ge d(v)] \]
is equivalent to determining if the threshold $\tau$ is greater than or equal to some $v$, and hence by the definition of $s_{\lambda^{t-1}}(x,v)$ in Definition \ref{def:threshold} and of the best response $h$ in Definition \ref{def:BR_h}, if $v$ is set to $f(x)$ it follows that \[h(x) = s_{\lambda^{t-1}}(x, f(x)).\]
\end{proof}

\finalerror*

In order to prove this, we will proceed through the specifics of each line of the proof sketch in the section 3.2 through
Lemmas \ref{lem:eqfirst} through \ref{lem:error-closeness}.

\begin{lemma}[Equality in Equation \ref{eq:first}]
\label{lem:eqfirst}
$$ err(h^*) = L^*(h^*, \lambda^*)$$
\end{lemma}

\begin{proof}
Consider the optimal solution $(h^*, \lambda^*)$ to $\psi(f^*, \gamma, \cH)$, and recall that $\err(h) = \E_{x \sim \cD_\cX} [f^*(x)\ell(h(x), 1) + (1-f^*(x))\ell(h(x),0)]$. Since the solution is optimal, it follows from complementary slackness, for each group $g$ one of the following must hold: Either the constraint is exactly tight and so its ``violation'' term in the Lagrangian evaluates to 0, or its corresponding dual variables $\lambda^\pm_g = 0$. Thus, $L_f^*(h^*, \lambda^*)$ simplifies to 

\begin{align*}
  L_{f}^*(h^*,\lambda^*)  &=
\E_{x \sim \cD_\cX} \bigg[
f(x)\ell(h(x), 1) + (1-f(x))\ell(h(x),0) \\ 
& \quad + 0 \cdot \sum_{g \in \cG}\lambda_g^+ \big( \ell(h(x),0)g(x)(1-f(x)) - \beta_g \ell(h(x),0)(1-f(x)) - \gamma \big) \\
& \quad + 0 \cdot \sum_{g \in \cG} \lambda_g^-\big(\beta_g \ell(h(x),0)(1-f(x)) - \ell(h(x),0)g(x)(1-f(x)) - \gamma \big) \bigg] \\
&= \E_{x \sim \cD_\cX} \bigg[f(x)\ell(h(x), 1) + (1-f(x))\ell(h(x),0)\bigg]\\
&= \err(h^*)
\end{align*}

\end{proof}

\begin{lemma}[Bounding Equation \ref{eq:first} by Equation \ref{eq:second}]
\label{lem:eqsecond}
$$ L^*(h^*, \lambda^*) \ge L^*(h^*, \hat\lambda).$$
\end{lemma}

\begin{proof}
This follows from the dual optimality condition that $\lambda^* \in \arg\max_\lambda L^*(h^*, \lambda)$.
\end{proof}

\begin{lemma}[Bounding Equation \ref{eq:second} by Equation \ref{eq:third}]
\label{lem:Lagr-closeness}
Fix any $\lambda$. If $\hat{f}$ is $\alpha$-multicalibrated with respect to $\cG, \cH,$ and $\cG \times \cH = \{g(x) \cdot h(x) \vert g \in \cG, h \in \cH\}$, then  then we have
\[ \left|\hat{L}(h^*,\lambda) - L^*(h^*, \lambda)\right| \le  \alpha(3 + 2\Vert \lambda \Vert_1).
\]
\end{lemma}

\begin{proof}
Observe that we can write:
\[
\hat L(h,\lambda) =  L_1(h, \lambda) - \gamma \sum_{g \in \cG} (\lambda_g^+ + \lambda_g^-) - \hat L_2(h, \lambda),
\]
where 
\begin{align*}
    L_1(h, \lambda) &= \E_{x \sim \cD_\cX} \Bigg[\ell(h(x), 0) \Big( 1 + \sum_{g \in \cG} \lambda_g (g(x)-\beta_g) \Big) \Bigg], \\
    \hat L_2(h, \lambda) &= \E_{x \sim \cD_\cX} \Bigg[\hat f(x) \Big( - \ell(h(x), 1) + \ell(h(x), 0) \big(1 + \sum_{g \in \cG} \lambda_g (g(x)-\beta_g)\big)\Big) \Bigg].
\end{align*}

Similarly, we can write:
\[
L^*(h,\lambda) =  L_1(h, \lambda) - \gamma \sum_{g \in \cG} (\lambda_g^+ + \lambda_g^-) - L_2^*(h, \lambda),
\]
where 
\[
L_2^*(h, \lambda) = \E_{x \sim \cD_\cX} \Bigg[f^*(x) \Big( - \ell(h(x), 1) + \ell(h(x), 0) \big(1 + \sum_{g \in \cG} \lambda_g (g(x)-\beta_g)\big)\Big) \Bigg].
\]
Observe that the $L_1$ term does not depend on $\hat f$ or $f^*$ and so is common between $\hat L$ and $L^*$. We can bound $\hat L_2$ as follows:
\begin{align*}
\hat{L}_2(h^*, \lambda) &= \E_{x \sim \cD_\cX} \Bigg[\hat f(x) \Big( - \ell(h^*(x), 1) + \ell(h^*(x), 0) \big(1 + \sum_{g \in \cG} \lambda_g (g(x)-\beta_g)\big)\Big) \Bigg] \\
&= \E_{x \sim \cD_\cX} \Bigg[\hat f(x) \Big( - (1-h^*(x)) + h^*(x) \big(1 + \sum_{g \in \cG} \lambda_g (g(x)-\beta_g)\big)\Big) \Bigg] \\
&= \sum_{v \in R} \quad \Pr[\hat f(x) = v] \E_{x \sim \cD_x} \Bigg[\hat f(x) \Big( - (1-h^*(x)) + h^*(x) \big(1 + \sum_{g \in \cG} \lambda_g (g(x)-\beta_g)\big)\Big) \Bigg \vert \hat f(x) = v \Bigg] \\
&\le \sum_{v \in R} \quad \Pr[\hat f (x) = v] \E_{x \sim \cD_x} \Bigg[f^*(x) \Big( - (1-h^*(x)) + h^*(x) \big(1 + \sum_{g \in \cG} \lambda_g (g(x)-\beta_g)\big)\Big) \Bigg\vert \hat f(x) = v \Bigg] \\
&\quad + \alpha\left(3 + \sum_{g\in\cG} \lambda_g(1+\beta_g)\right) \\
&\leq L_2^*(h^*, \lambda) + \alpha \left(3 + 2\Vert \lambda\Vert_1\right),
\end{align*}
\noindent where the first inequality follows from the fact that $h^* \in \cH$ and $\hat f$ is multicalibrated with respect to $\cG, \cH,$ and $\cG \times \cH$, which we verify below: 
\begin{align*}
    \sum_{v \in R} & \quad \Pr[\hat f(x) = v] \E_{x \sim \cD_x} \Bigg[\bigg(f^*(x) - \hat f(x)\bigg) \cdot \Big( - (1-h^*(x)) + h^*(x) \big(1 + \sum_{g \in \cG} \lambda_g (g(x)-\beta_g)\big)\Big) \Bigg\vert \hat f(x) = v \Bigg] \\
    &= \sum_{v \in R} \Pr[\hat f(x) = v] \Bigg[\bigg(f^*(x) - \hat f(x)\bigg) \cdot \Big( - 1 + 2 h^*(x) + h^*(x) \sum_{g \in \cG} \lambda_g (g(x)-\beta_g)\Big) \Bigg\vert \hat f(x) = v \Bigg] \\
    &= - \sum_{v \in R} \Pr[\hat f(x) = v] \E_{x \sim \cD_x} \Big[\hat f^*(x) - \hat f(x) \big\vert \hat f(x) = v \Big] \\
    &\quad + 2 \sum_{v \in R} \Pr[\hat f(x) = v] \E_{x \sim \cD_x} \Big[ (f^*(x) - \hat f(x))  h^*(x) \big\vert \hat f(x) = v \Big] \\
    &\quad + \sum_{v \in R} \Pr[\hat f(x) = v] \sum_{g \in \cG} \lambda_g\E_{x \sim \cD_x} \Big[(f^*(x) - \hat f(x))  h^*(x) g(x) \big\vert \hat f(x) = v \Big] \\
    &\quad - \sum_{v \in R} \Pr[\hat f(x) = v] \sum_{g \in \cG} \lambda_g \beta_g \E_{x \sim \cD_x} \Big[(f^*(x) - \hat f(x))  h^*(x) \big\vert \hat f(x) = v \Big] \\
    &\le 3\alpha + \sum_{g\in\cG} \lambda_g(1+\beta_g)\alpha \\
    &\le 3\alpha + \alpha \sum_{g \in \cG} \lambda_g(1 + \max_{g' \in \cG} \beta_{g'}) \\
    &\le 3\alpha + \alpha \sum_{g \in \cG} \lambda_g(1 + 1) \\
    &\le 3\alpha + 2\Vert \lambda \Vert_1 \alpha
\end{align*}
Similarly, we can show that   $L^*(h^*, \lambda) - \hat{L}(h^*, \lambda) \le  \alpha \left(3 + 2\Vert \lambda\Vert_1\right)$. Putting everything together, we get that:

$$ \left\vert\hat L(h^*, \lambda) - L^*(h^*, \lambda) \right\vert \le \alpha(3 + 2\Vert \lambda \Vert_1).$$



This concludes the proof.
\end{proof}

\begin{lemma}[Bounding Equation \ref{eq:third} by Equation \ref{eq:fourth}]
\label{lem:lhats}
\[\hat L(h^*, \hat \lambda) \ge \hat L(\hat h, \hat \lambda)\]
\end{lemma}

\begin{proof}
This follows from the primal optimality condition that $\hat h \in \arg\min_{h \in \cH_A} \hat L(h, \hat \lambda)$ and that $\cH \subseteq \cH_A$. 
\end{proof}

\begin{lemma}[Equality of Equation \ref{eq:fourth} and Equation \ref{eq:fifth}]
\label{lem:LhattoErr}
\[ 
\hat L(\hat h, \hat \lambda) = \widehat \err(\hat h)
\]
\end{lemma}

\begin{proof}
This follows the same complimentary slackness argument as the proof of Lemma \ref{lem:eqfirst}.
\end{proof}

\begin{lemma}[Bound of Equation \ref{eq:fifth} by Equation \ref{eq:barh}]
\label{lem:barh}
Consider $\bar h$ output by algorithm \ref{alg:descent} after $T = \frac{1}{4} \cdot C^2 \cdot \left(C^2 + 4|\cG| \right)^2 $ rounds. Then, 
$$ \widehat \err(\hat h) + 2/C \ge \widehat \err(\bar h) $$
\end{lemma}

\begin{proof}
This follows directly from Theorem \ref{thm:alg-main}.
\end{proof}

\begin{lemma}[Bound of Equation \ref{eq:barh} by Equation \ref{eq:sixth}]
\label{lem:error-closeness}
Let $\hat f$ be $\alpha$-approximately jointly multicalibrated with respect to $\cB$. Then,
$$ \left\vert \widehat \err (\bar h) - \err(\bar h) \right\vert \le 2\alpha.$$
\end{lemma}
\begin{proof}
Since $\bar h$ is a randomized model that mixes uniformly over model $\hat h_t$ for $t \in [T]$, it suffices to show that for every $t \in [T]$, 

\begin{align*}
    \left|\widehat{\err}(\hat h_t) - \err(\hat h_t)\right| \le 2\alpha.
\end{align*}
We can compute:

\begin{align*}
\widehat{\err}(\hat{h}_t) &= \E_{x \sim \cD_\cX} \left[\hat{f}(x)\ell( \hat{h}_t, 1) + (1-\hat{f}(x))\ell(\hat{h}_t(x),0)\right], \\
&= \sum_{v \in R} \Pr[\hat f(x)=v, s_{\lambda^{t-1}}(x,v)=0]  \E_{x \sim \cD_\cX} [\hat f(x) \ell(\hat{h}_t(x), 1) + (1-\hat f(x))\ell(\hat{h}_t(x),0) | \hat f(x) = v, s_{\lambda^{t-1}}(x,v)=0], \\
&+ \sum_{v \in R} \Pr[\hat f(x) =v, s_{\lambda^{t-1}}(x,v)=1]  \E_{x \sim \cD_\cX} [\hat{f}(x) \ell(\hat{h}_t(x), 1) + (1-\hat f(x))\ell(\hat{h}_t(x),0) | \hat f(x) = v, s_{\lambda^{t-1}}(x,v)=1].
\end{align*}

By Lemma \ref{lem:fixh}, $\hat h_t(x) = s_{\lambda^{t-1}}(x, \hat f(x))$, and so in particular conditioning on   $\hat{f}(x)=v$ and $s_{\lambda^{t-1}}(x,v)$ fixes the value of $\hat h_t(x)$.  So, we can rewrite the above as 

\begin{align*}
\widehat{\err}(\hat{h}_t)  &= \sum_{v \in R} \Pr[\hat f(x) =v, s_{\lambda^{t-1}}(x,v)=0]  \E_{x \sim \cD_\cX} [ \hat{f}(x) | \hat f(x) = v, s_{\lambda^{t-1}}(x,v)=0]\\
&+ \sum_{v \in R} \Pr[\hat f(x)=v, s_{\lambda^{t-1}}(x,v)=1]  \E_{x \sim \cD_\cX} [1-\hat{f}(x) | \hat{f}(x) = v, s_{\lambda^{t-1}}(x,v)=1]\\
&\le \sum_{v \in R} \Pr[\hat f(x)=v, s_{\lambda^{t-1}}(x,v)=0]  \E_{x \sim \cD_\cX} [ f^*(x) | \hat{f}(x) = v, s_{\lambda^{t-1}}(x,v)=0] + \alpha\\
&+ \sum_{v \in R} \Pr[\hat f(x)=v, s_{\lambda^{t-1}}(x,v)=1]  \E_{x \sim \cD_\cX} [1-f^*(x) | \hat{f}(x) = v, s_{\lambda^{t-1}}(x,v)=1] +\alpha\\
&= \E_{x \sim D_\cX}[f^*(x) \ell(\hat{h}_t(x), 1) + (1-f^*(x))\ell(\hat{h}_t(x), 0)] + 2\alpha\\
&=\err(\hat{h}_t) + 2\alpha,
\end{align*}
where the inequality comes from our $\alpha$-approximate joint multicalibration guarantee. The same argument yields the opposite direction, so we are done.
\end{proof}

We now have the tools to prove our main theorem.

\begin{proof}[Proof of Theorem \ref{thm:final-error}]
Applying lemmas \ref{lem:eqfirst} through \ref{lem:error-closeness} gives us 
\begin{eqnarray}
\err(h^*) &=& L^*(h^*,\lambda^*) \quad ( \text{Lemma \ref{lem:eqfirst}})\\
&\geq& L^*(h^*, \hat \lambda)  \quad ( \text{Lemma \ref{lem:eqsecond}}) \\
&\geq& \hat L(h^*, \hat \lambda) -  \alpha(3 + 2\Vert \lambda \Vert_1) \quad ( \text{Lemma \ref{lem:Lagr-closeness}}) \\
&\geq& \hat L(\hat h, \hat \lambda) -  \alpha(3 + 2\Vert \lambda \Vert_1) \quad ( \text{Lemma \ref{lem:lhats}}),
\end{eqnarray}
and
\begin{eqnarray}
\hat{L}(\hat h, \hat \lambda) &=& \widehat \err(\hat h) \quad (\text{Lemma \ref{lem:LhattoErr}})\\
&\ge& \widehat \err(\bar h) - 2/C \quad (\text{Lemma \ref{lem:barh}}) \\
&\ge& \err(\bar h) - 2/C - 2\alpha \quad  (\text{Lemma \ref{lem:error-closeness}}).
\end{eqnarray}

Putting this all together gives us 
\begin{align*}
\err(h^*) &\ge \err(\bar h) -  \alpha(3 + 2\Vert \lambda \Vert_1) - 2/C - 2\alpha\\
    &= \err(\bar h) - \alpha(5+2\Vert \lambda \Vert_1)-2/C\\
    &\ge \err(\bar h) - \alpha(5+2C)-2/C
\end{align*}

 We want to set $C$ to minimize this discrepancy. Noting that the derivative of $ \alpha(5+2C)+2/C$ with respect to $C$ is $2\alpha-2/C^2$, we get a minimization at $C=\sqrt{1/\alpha}$.


Setting $C$ as such gives the desired bound:

\begin{align*}
    \err(h^*) &\geq \err(\bar h) - \alpha(5 + 2\sqrt{1/\alpha})-2\sqrt{\alpha}.\\
\end{align*}

    Following a similar analysis as Lemma \ref{lem:error-closeness}, we can bound the fairness constraints on $\bar{h}$ by bounding them for the model $\hat h_t$ found at every round $t \in [T]$ of algorithm \ref{alg:descent}.
    \begin{align*}
    \hat{\rho}_g&(\hat h_t) - \hat{\rho}(\hat h_t) = \E_{x \sim \cD_x} [ (1 - \hat{f}(x)) \ell(\hat{h}_t(x), 0) g(x) ] -  \E_{x \sim \cD_x} [ (1 - \hat{f}(x)) \ell(\hat{h}_t(x), 0) ] \\
    &= \sum_{v \in R} \Pr[\hat{f}(x) = v, s_{\lambda^{t-1}}(x, v) = 0] \E_{x \sim \cD_x} [ (1 - \hat{f}(x)) \ell(\hat{h}_t(x), 0) \cdot (g(x) - 1)  | \hat{f}(x) = v, s_{\lambda^{t-1}} (x, v) = 0] \\ 
    & \quad + \Pr[\hat{f}(x) = v, s_{\lambda^{t-1}}(x, v) = 1] \E_{x \sim \cD_x} [ (1 - \hat{f}(x)) \ell(\hat{h}_t(x), 0) \cdot (g(x) - 1)  | \hat{f}(x) = v, s_{\lambda^{t-1}} (x, v) = 1] \\
    &= \sum_{v \in R} \Pr[\hat{f}(x) = v, s_{\lambda^{t-1}, \geq}(x, v) = 1] \E_{x \sim \cD_x} [ (1 - \hat{f}(x)) \ell(\hat{h}_t(x), 0) \cdot (g(x) - 1)  | \hat{f}(x) = v, s_{\lambda^{t-1}} (x, v) = 1] \\
    &\leq \sum_{v \in R} \Pr[\hat{f}(x) = v, s_{\lambda^{t-1}}(x, v) = 1] \E_{x \sim \cD_x} [ (1 - f^*(x)) \ell(\hat{h}_t(x), 0) \cdot (g(x) - 1)  | \hat{f}(x) = v, s_{\lambda^{t-1}} (x, v) = 1] + \alpha \\
    &= \E_{x \in \cD_x} [ (1 - f^*(x)) \ell(h_t(x), 0) \cdot (g(x) - 1) ] + \alpha \\
    &= \rho_g(h_t) - \rho(h_t) + \alpha.
    \end{align*}
    Here, the inequality comes from our multicalibration guarantees. We can repeat the same argument in the opposite direction, and get that 
    \[
     w_g \left \vert \rho_g(h^*) - \rho(h^*) \right\vert \ge w_g \left \vert \rho_g(\bar h) - \rho(\bar h) \right\vert - w_g \alpha.
    \]
\end{proof}
\fi

\section{Achieving Joint Multicalibration}
\label{ap:jointmultical}

In this section we give an algorithm that can take as input any model 
 $f: \cX \to [0,1]$ and transform it into a new model $\hat{f}:\cX \to R$ such that $\hat f$ achieves multicalibration in expectation with respect to a class of functions $\cC_1 \subset \{0,1\}^\cX$  and simultaneously, joint multicalibration in expectation with respect to a class of functions $\cC_2 \subset \{0,1\}^{\cX \times R}$ where $R=\{0, \frac{1}{m}, \frac{2}{m}, \dots, 1\}$ for some $m>0$. Our algorithm can be viewed as a variant of the original multicalibration algorithm of \cite{multicalibration} (our variant achieves the stronger guarantee of  calibration in expectation, first defined in \cite{omnipredictors}), or a simplification of the split-and-marge algorithm of \cite{omnipredictors}, which replaces the ``merge'' operation with simple per-update rounding. 
 
 First we observe that without loss of generality, we can focus on achieving joint multicalibration for a single class of functions. To see this, note that given $\cC_1 \subset \{0,1\}^\cX$,  we can transform it into an identical class of two argument functions that simply ignore their second argument:
\[
    \cC'_2 = \{\text{$c$ where $c(x,v)=c_1(x)$ for every $c_1 \in \cC_1$}\}.
\]

Note that if $\hat{f}$ is $\alpha$-approximately joint-multicalibrated with respect to $\cC'_2$, then it is $\alpha$-approximately multicalibrated with respect to $\cC_1$ and vice versa. In other words, in order to be simultaniously multicalibrated with respect to $\cC_1$ \emph{and} joint-multicalibrated with respect to $\cC_2$, it is sufficient (actually equivalent) to be joint-multicalibrated with respect to $\cC_2 \cup \cC'_2$. Therefore, we focus on enforcing joint-multicalibration with respect to arbitrary $\cC \subset \{0,1\}^{\cX \times [0,1]}$.

Before we describe the algorithm, we define the round operation. Write $[\frac{1}{m}]=\{0, \frac{1}{m}, \frac{2}{m}, \dots, 1\}$ for any $m>0$. We let $f'=Round(f,m)$ to denote the function that simply rounds the output of $f$ to the nearest grid point of $[1/m]$. Similarly, we write $Round(v,m) = \arg\min_{v' \in [1/m]}|v'-v|$ to denote the grid point of $[\frac{1}{m}]$ closest to $v$.

\begin{algorithm}
\caption{Multicalibration algorithm}\label{alg:postprocess-to-multicalibrate}
 \KwInput{$\alpha, f, \cC$}
 
$m=\frac{1}{\alpha}$\\
$f_0 = Round(f, m)$\\
$t=0$\\
\While{there exists a $c \in \cC$ such that: \[
    \sum_{v \in R} \Pr_{x \sim \cD_\cX}[f_t(x) = v, c(x, v)=1] \left(v - \Pr_{(x,y) \sim \cD}[y|f_t(x) = v, c(x, v)=1]\right)^2 \ge \alpha
\] }{ 
    Let
    \begin{align*}
       (v_t, c_t) &= \arg\max_{v \in R, c \in \cC} \Pr_{x \sim \cD_\cX}[f_t(x) = v, c(x, v)=1] \cdot \left(v - \Pr_{(x,y) \sim \cD}[y|f_t(x) = v, c(x, v)=1]\right)^2\\
       S_t &= \{x \in \cX: f_t(x) = v, c_t(x, v_t)=1\}\\
       \tilde{v}_t &= \E_{(x,y) \sim \cD}[y|x \in S_t]\\
        v'_t &= Round(\tilde{v}_t, m)
    \end{align*}
    Let \[
        f_{t+1}(x) = \begin{cases}
            v'_t \quad\text{if $x \in S_t$}\\
            f_t(x) \quad\text{otherwise}.
        \end{cases}
    \]
    
    $t = t+1$
 }
\end{algorithm}
\begin{theorem}
  The output of Algorithm~\ref{alg:postprocess-to-multicalibrate} $f_T: \cX \to \{0, \alpha, 2\alpha, \dots, 1\}$ is $\sqrt{\alpha}$-approximately jointly multicalibrated with respect to $\cC$ where $T \le \frac{4}{\alpha^2}$.
\end{theorem}
\begin{proof}
  By definition, the output of the algorithm $f_T$ is such that 
  \[
    \sum_{v \in R} \Pr_{x \sim \cD_\cX}[f_t(x) = v, c(x, v)=1] \left(v - \Pr_{(x,y) \sim \cD}[y|f_t(x) = v, c(x, v)=1]\right)^2 < \alpha
  \]
  for every $c \in \cC$, meaning it is satisfies $\sqrt{\alpha}$-joint calibration:
  \begin{align*}
  &\sum_{v \in R} \Pr_{x \sim \cD_\cX}[f_t(x) = v, c(x, v)=1] \left|v - \Pr_{(x,y) \sim \cD}[y|f_t(x) = v, c(x, v)=1]\right| \\
  &\le \sqrt{\sum_{v \in R} \Pr_{x \sim \cD_\cX}[f_t(x) = v, c(x, v)=1] \left(v - \Pr_{(x,y) \sim \cD}[y|f_t(x) = v, c(x, v)=1]\right)^2} \\
  &< \sqrt{\alpha}.
  \end{align*}
  
  So it suffices to show that the algorithm halts in less than $T \le \frac{4}{\alpha^2}$ rounds. Define \[
    B(f) = \E_{(x,y) \sim \cD}[(y- f(x))^2].
  \]
  We use $B$ as a potential function and show that we decrease it in each round in the following lemma.
  \begin{lemma} 
  \label{lem:alg-multicalibr-progress}
  For every $t < T$,
  $B(f_{t+1}) - B(f_t) \le -\frac{\alpha^2}{4}$
  \end{lemma}
    \begin{proof}
    Define $\tilde{f}_{t}$ such that
    \[
    \tilde{f}_{t}(x) = \begin{cases}
    \tilde{v}_{t} \quad&\text{if $x \in B_t$}\\ f_t(x) \quad&\text{otherwise}.
    \end{cases}
    \]
    \begin{align*}
        B(f_{t+1}) - B(f_t) = \underbrace{\left(B(f_{t+1})- B(\tilde{f}_{t})\right)}_{(*)}  + \underbrace{\left(B(\tilde{f}_{t}) - B(f_t)\right)}_{(**)}
    \end{align*}
    
    \item\paragraph{Bounding (*):}
    \begin{align*}
        B(f_{t+1})- B(\tilde{f}_{t}) &= \Pr_{x \sim \cD_\cX}[x \in S_t] \cdot \E_{(x,y) \sim \cD}[(y-f_{t+1}(x))^2 - (y-\tilde{f}_t(x))^2 | x \in S_t]\\
        &=\Pr_{x \sim \cD_\cX}[x \in S_t] \cdot \E_{(x,y) \sim \cD}[((y-\tilde{f}_{t}(x)) + (\tilde{v}_t - v'_t))^2 - (y-\tilde{f}_t(x))^2 | x \in S_t]\\
        &=\Pr_{x \sim \cD_\cX}[x \in S_t] \cdot \E_{(x,y) \sim \cD}[2(y-\tilde{v}_t)(\tilde{v}_t - v'_t) + (\tilde{v}_t - v'_t)^2| x \in S_t]\\
        &\le \Pr_{x \sim \cD_\cX}[x \in S_t] \cdot  \frac{1}{4m^2}
    \end{align*}
    where the last inequality follows from the fact that $\tilde{v}_t = \E_{(x,y)\sim \cD}[y|x \in S_t]$ and $|\tilde{v}_t - v'_t| \le \frac{1}{2m}$.
    
    \item\paragraph{Bounding (**):}
    Because in round $t$, \[
        \sum_{v \in R} \Pr_{x \sim \cD_\cX}[f_t(x) = v, c(x, v)=1] \left(v - \Pr_{(x,y) \sim \cD}[y|f_t(x) = v, c(x, v)=1]\right)^2 \ge \alpha,
    \]
    we must have 
    \[
        \Pr_{x \sim \cD_\cX}[x \in S_t] \left(v_t - \tilde{v}_t \right)^2 = \Pr_{x \sim \cD_\cX}[x \in S_t] \left(v_t - \Pr_{(x,y) \sim \cD}[y|x \in S_t]\right)^2 \ge \frac{\alpha}{m+1}.
    \]
    
    Now,we show that
    \begin{align*}
        B(\tilde{f}_t) - B(f_{t+1}) &= \Pr_{x \sim \cD_\cX}[x \in S_t] \cdot \E_{(x,y)\sim \cD}[(y-\tilde{f}_t(x))^2 - (y-f_t(x))^2 | x \in S_t]\\
        &= \Pr_{x \sim \cD_\cX}[x \in S_t] \cdot \E_{(x,y)\sim \cD}[(y-\tilde{f}_t(x))^2 - ((y-\tilde{f}_t(x)) + (\tilde{v}_t - v_t))^2 | x \in S_t]\\
        &=\Pr_{x \sim \cD_\cX}[x \in S_t] \cdot \E_{(x,y)\sim \cD}[-2(y-\tilde{v}_t)(\tilde{v}_t - v_t) - (\tilde{v}_t - v_t))^2 | x \in S_t]\\
        &\le \frac{-\alpha}{m+1}
    \end{align*}
    where the last inequality follows from the fact that $\E_{(x,y)}[y | x\in S_t] = \tilde{v}_t$.
    
    Combining them together, we get
    \begin{align*}
        B(f_{t+1}) - B(f_t) &\le \frac{1}{4m^2} - \frac{\alpha}{m+1}\\
        &=\frac{\alpha^2}{4}-\frac{\alpha^2}{\alpha+1}\\
        &\ge \frac{\alpha^2}{4} - \frac{\alpha^2}{2}\\
        &=-\frac{\alpha^2}{4}.
    \end{align*}
    \end{proof}
    
    Iterating Lemma~\ref{lem:alg-multicalibr-progress} over $T$ rounds, we have \[
    B(f_T) \le B(f_0) - T\frac{\alpha^2}{4}. \]
    Also, because $B(f) \in [0,1]$ for any $f$, it must be that $T \le \frac{4}{\alpha^2}$.
\end{proof}

\section{Out of Sample Guarantees}
\label{ap:generalization}

In the body of the paper, we assumed that we had direct access to distributional quantities --- in particular, we needed to evaluate expectations over the feature distribution. In this section, we show that it is possible to estimate these quantities from modest amounts of unlabeled data sampled from the underlying distribution, and that the guarantees of our algorithm carry over to the underlying distribution. In particular, our algorithm results in a solution to the linear program that approximately satisfies its constraints on the underlying distribution, and achieves objective value that is approximately optimal within its comparison class. The strategy we take is to analyze a slightly modified algorithm (Algorithm \ref{alg:empirical}), which at every stage, uses a fresh sample of data to evaluate the necessary expectations empirically. In particular, it uses a \emph{new} sample at every iteration, and so has sample complexity that scales linearly with the number of iterations. Using techniques from adaptive data analysis \cite{dwork2015preserving,bassily2016algorithmic,jung2021new} similar to how they are used by \cite{multicalibration} to prove sample complexity bounds, we could reduce our linear dependence on $T$ in our sample complexity bound by a quadratic factor by reusing data across rounds, but we settle for the conceptually simpler bound here. 


\begin{theorem}\label{thm:generalization}
Fix any distribution $\cD$, hypothesis class $\cH$,  class of group indicators $\cG$, dual bound $C$, and $\epsilon,\delta > 0$. After $T$ rounds, with probability $1-\delta$, Algorithm \ref{alg:empirical} outputs a randomized hypothesis $\bar{h}$ such that $err(\bar{h}) \leq \text{OPT} + \frac{2}{C} + 8 \epsilon$ and $\omega_g | \rho_g(\bar{h}) - \rho(\bar{h})| \leq \gamma + \frac{1}{C} + \frac{2}{C^2} + \frac{8\epsilon}{C}$, where $\text{OPT}$ is the objective value of the optimal solution of $\psi(f, \gamma, \cH_A)$. It makes use of $m = O\left(T \frac{\log(\frac{2 T |G| }{\delta})}{2\epsilon^2}\right)$ samples of unlabeled data drawn i.i.d. from $\cD_{\cX}$. Here $T$ is as specified in the algorithm: 
 $ T = \frac{1}{4} \cdot C^2 \cdot (C^2 + 4|\cG|)^2 $.
\end{theorem}
\begin{lemma}\label{lemma:per-round}
Fix any distribution $\cD$, hypothesis class $\cH$, and class of group indicators $\cG$. In a single round $t$ of Algorithm \ref{alg:empirical} with $S_t \sim \cD^m$ for $m = O( \frac{\log(\frac{2 |G| }{\delta})}{2\epsilon^2})$, Algorithm \ref{alg:empirical} returns a hypothesis $h_t$ that with probability $1 - \delta$ 
satisfies for all $g \in G$:
\begin{align*}
|err(h_t, g, \cD) - err(h_t, g, S_t)| &\leq \epsilon \\
| \rho(h_t, g, \cD) - \rho(h_t, g, S_t) | &\leq \epsilon.
\end{align*}
\end{lemma}

\begin{algorithm}\label{alg:empirical}
\caption{Projected Gradient Descent Algorithm}
    \KwInput{$\cD$: data distribution, $f: \cX \to [0, 1]$: regression function, $\cG$: groups, $\gamma$: tolerance on fairness violation, $C$: bound on dual $(\|\lambda\|_1 \leq C)$, $\eta$: learning rate, m = $\frac{\log(\frac{2 |G| }{\delta})}{2\epsilon^2}$: batch size of fresh data for each round of gradient descent, $\epsilon$: per round estimation error, $\delta$: failure probability}
    
    Initialize dual vector $\lambda^0 = {\bf 0}$ and set $ T = \frac{1}{4} \cdot C^2 \cdot (C^2 + 4|\cG|)^2 $.
    
    \For{$t = 1, \ldots, T$}
    {
    Primal player updates $h_t$
    \[
        h_t(x) = \begin{cases}
        1, & \text{if }  f(x) \geq \frac{1 + \sum_{g \in \cG} \lambda^{t-1}_g (g(x)-\beta_g)}{2 + \sum_{g \in \cG} \lambda^{t-1}_g (g(x)-\beta_g)} \text{ and } (2 + \sum_{g \in \cG} \lambda^{t-1}_g (g(x)-\beta_g)) > 0,\\
        0, & \text{if } f(x) < \frac{1 + \sum_{g \in \cG} \lambda^{t-1}_g (g(x)-\beta_g)}{2 + \sum_{g \in \cG} \lambda^{t-1}_g (g(x)-\beta_g)} \text{ and } (2 + \sum_{g \in \cG} \lambda^{t-1}_g (g(x)-\beta_g)) > 0,\\
        1, & \text{if } f(x) \leq \frac{1 + \sum_{g \in \cG} \lambda^{t-1}_g (g(x)-\beta_g)}{2 + \sum_{g \in \cG} \lambda^{t-1}_g (g(x)-\beta_g)} \text{ and } (2 + \sum_{g \in \cG} \lambda^{t-1}_g (g(x)-\beta_g)) < 0,\\
        0, & \text{if } f(x) > \frac{1 + \sum_{g \in \cG} \lambda^{t-1}_g (g(x)-\beta_g)}{2 + \sum_{g \in \cG} \lambda^{t-1}_g (g(x)-\beta_g)} \text{ and } (2 + \sum_{g \in \cG} \lambda^{t-1}_g (g(x)-\beta_g)) < 0.
        \end{cases}
    \]
    Sample $S_t$ i.i.d. from $\cD^m$
    Compute
    \begin{align*}
        \hat \rho^t_g& = \E_{(x, y) \sim S_t} [\ell(h_t(x),0)g(x)(1- f(x))] \text{ for all } g \in \cG,\\
        \hat \rho^t &= \E_{(x, y) \sim S_t} [\beta_g \ell(h_t(x),0)(1-  f(x))], \text{ where } \beta_g = \Pr[g(x) = 1 | y = 0]
    \end{align*}
    Dual player updates 
    \begin{align*}
    \lambda_{g}^{t, +} &= \max(0, \lambda_g^{t,+} + \eta \cdot (\hat \rho^t_g - \hat \rho^t - \gamma)), \\
    \lambda_{g}^{t, -} &= \max(0, \lambda_g^{t,-} + \eta \cdot (\hat \rho^t - \hat \rho^t_g - \gamma)).
    \end{align*}
    
    Dual player sets $\lambda^t =\sum_{g \in \cG} \lambda_g^{t, +} - \lambda_g^{t, -}$.
    
    If $\|\lambda^t\|_1 > C$, set $\lambda^t = C \cdot \frac{\lambda^t}{ \|\lambda^t\|_1}$.
    }
    \KwOutput $\bar{h} := \frac{1}{T} \sum_{t=1}^T \hat{h}_t$, a uniformly random classifier over all rounds' hypotheses.
\end{algorithm}

\begin{theorem}[Chernoff-Hoeffding Bound] Let $X_1, X_2, \ldots, X_m$ be i.i.d. random variables with $a \leq X_i \leq b$ and $\mathbb{E}[X_i] = \mu$ for all $i$. Then, for any $\alpha > 0,$
\begin{align*}
    \Pr\left(\left\vert\frac{\sum_{i}X_i}{m} - \mu\right\vert > \alpha\right) \leq 2 \exp \left( \frac{-2\alpha^2m}{(b-a)^2}\right).
\end{align*}
\end{theorem}

\begin{proof}[Proof of Lemma \ref{lemma:per-round}]
This claim follows by applying a Chernoff-Hoeffding bound with $m \geq \frac{\ln(\frac{2|G|}{\delta})}{2\epsilon^2}$
\end{proof}
\begin{proof}[Proof Sketch of Theorem \ref{thm:generalization}]
Taking $m > \frac{\log(\frac{2 T |G| }{\delta})}{2\epsilon^2}$, we have that in a single round $t$ of our algorithm we are able to estimate the true distributional classification and fairness constraint errors up to an additive error of $\epsilon$ with probability $1 - \delta / T$ --- and hence with probability $1-\delta$, we estimate these quantities up to additive error $\epsilon$ uniformly over all $T$ rounds. We can then make one small modification to the analysis of Algorithm \ref{alg:descent}. First observe that since the primal player's best response does not depend on any estimation of a distributional quantity based on the sample $S_t$, their regret is still zero, as it is in the analysis of Algorithm \ref{alg:descent}. The dual player, on the other hand, is given loss vectors that  deviate from the versions that would have been computed on the underlying distribution by at most $2\epsilon$ in $\ell_\infty$ norm, and hence experience additional regret (to the true distributional quantities) larger than in the analysis of Algorithm \ref{alg:descent} by up to an additional additive $4\epsilon$. 
Consequently, the equilibrium solution $(\bar{h}, \bar{\lambda})$ from Algorithm \ref{alg:empirical} is an $4\epsilon + 1/C$ approximate equilibrium to the zero-sum game of \ref{eq:game} which then, applying Theorem \ref{thm:approxminmax}, yields a $\frac{2}{C} + 8 \epsilon$ approximate solution to the objective of the original linear program.
\end{proof}

\section{Expanded Experimental Discussion}
\label{ap:experiments}

\begin{figure}
        \includegraphics[width=0.5\linewidth]{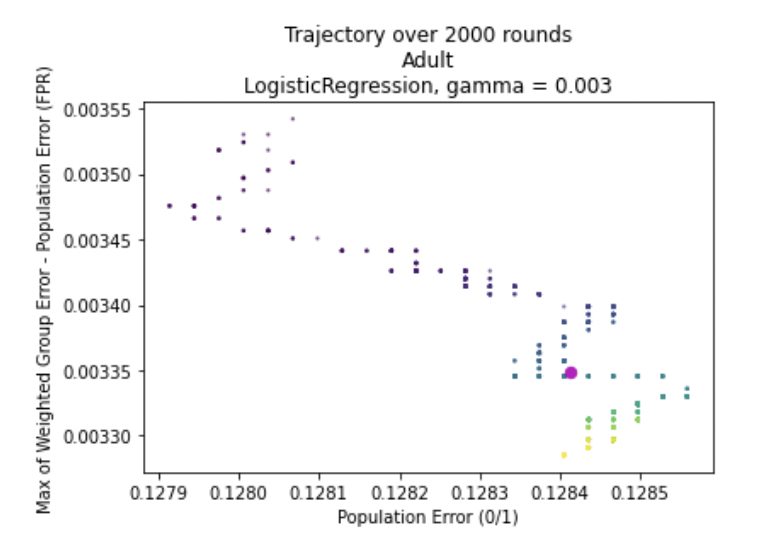}
        \includegraphics[width=0.5\linewidth]{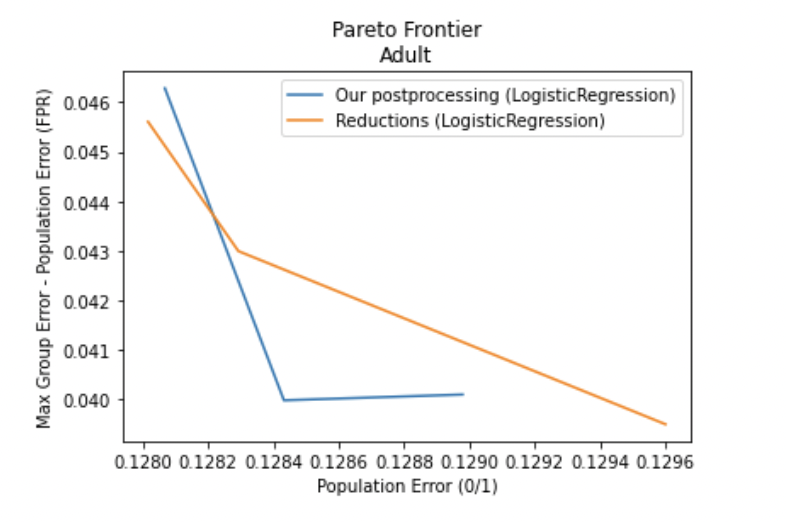}
        \caption{The plot on the left is a trajectory over 2000 iterations of gradient descent of our method post-processing a base logistic regression model, for a single value of $\gamma = 0.003$. The trajectory starts at the top of the figure and moves downwards with time, and the purple point represents the uniform distribution over the constituent models of the 2000 iterations. The plot on the right shows the Pareto curve for our method (blue) and for the fair reductions \citep{agarwal2018reductions} for constraint values ranging between $0.0025 \leq \gamma \leq  0.00355$.}
        \label{fig:folk}
\end{figure}

We provide additional experimental evaluation on the UCI Adult dataset \citep{Dua:2019}. The sensitive attributes we use are binary gender and race, categorized as White, Black, Asian and Pacific Islander, American Indian or Eskimo, and Other. Note that race and gender are intersecting attributes. In these experiments our algorithm is post-processing a standard sklearn logistic regression model, notably, as in our previous results, not guaranteed to be multicalibrated in the ways our theory requires. We also provide a comparison to the popular in-processing ``fair reductions'' method \citep{agarwal2018reductions}. We note that our algorithm performs competitively, even Pareto-dominating certain points on the reductions Pareto frontier. However, the reductions method is also able generate points corresponding to constraint violations that our method is not able to access --- this does not violate our theoretical findings, since we are not starting with a multicalibrated regression function. Our method requires solving a single logistic regression problem over the dataset (to compute the regression model $\hat f$ that we post-process), whereas the method of \cite{agarwal2018reductions} requires solving a  regression problem at every iteration. 

\end{document}